\documentclass{article}

\usepackage{microtype}
\usepackage{graphicx}
\usepackage{subfigure}
\usepackage{booktabs} % for professional tables

\usepackage{hyperref}

\usepackage[accepted]{icml2020}

\usepackage{mymath}

\icmltitlerunning{Obtaining Adjustable Regularization for Free via Iterate Averaging}

\begin{document}

\twocolumn[
\icmltitle{Obtaining Adjustable Regularization for Free via Iterate Averaging}



\icmlsetsymbol{equal}{*}

\begin{icmlauthorlist}
\icmlauthor{Jingfeng Wu}{jhu}
\icmlauthor{Vladimir Braverman}{jhu}
\icmlauthor{Lin F. Yang}{ucla}
\end{icmlauthorlist}

\icmlaffiliation{jhu}{Johns Hopkins University, Baltimore, MD, USA}
\icmlaffiliation{ucla}{University of California, Los Angeles, CA, USA}

\icmlcorrespondingauthor{Jingfeng Wu}{uuujf@jhu.edu}
\icmlcorrespondingauthor{Lin F. Yang}{linyang@ee.ucla.edu}

\icmlkeywords{Iterate Average, Regularization}

\vskip 0.3in
]



\printAffiliationsAndNotice{}  % leave blank if no need to mention equal contribution

\begin{abstract}
Regularization for optimization is a crucial technique to avoid overfitting in machine learning. In order to obtain the best performance, we usually train a model by tuning the regularization parameters. It becomes costly, however, when a single round of training takes significant amount of time. Very recently, \citet{neu2018iterate} show that if we run stochastic gradient descent (SGD) on linear regression problems, then by averaging the SGD iterates properly, we obtain a regularized solution. It left open whether the same phenomenon can be achieved for other optimization problems and algorithms. In this paper, we establish an averaging scheme that \emph{provably} converts the iterates of SGD on an arbitrary strongly convex and smooth objective function to its regularized counterpart with an \emph{adjustable} regularization parameter. Our approaches can be used for accelerated and preconditioned optimization methods as well. We further show that the same methods work empirically on more general optimization objectives including neural networks. In sum, we obtain \emph{adjustable} regularization \emph{for free} for a large class of optimization problems and resolve an open question raised by~\citet{neu2018iterate}.
\end{abstract}

\section{Introduction}
Regularization for optimization is a key technique for avoiding over-fitting in machine learning and statistics ~\cite{grandvalet2005semi,krogh1992simple,tibshirani1996regression,tikhonov1977solutions}.
The effects of explicit regularization methods, i.e., an extra regularization term added to the vanilla objective, are well studied, e.g., ridge regression~\cite{tikhonov1977solutions}, LASSO~\cite{tibshirani1996regression} and entropy regularization~\cite{grandvalet2005semi}.
Despite the great benefits of adopting explicit regularization, it could cause a huge computational burden to search for the optimal hyperparameter associated with the extra regularization term, especially for large-scale machine learning problems~\cite{devlin2018bert,He_2016,silver2017mastering}.

In another line of research, people recognize and utilize the implicit regularization caused by certain components in machine learning algorithms, e.g., initialization~\cite{he2015delving,hu2020provable}, batch normalization~\cite{ioffe2015batch,cai2018quantitative},
iterate averaging~\cite{bach2013non,jain2018parallelizing,neu2018iterate}, and optimizer such as gradient descent (GD)~\cite{gunasekar2018characterizing,soudry2018implicit,suggala2018connecting}.
The regularization effect usually happens
along the process of training the model and/or requires little post-computation.
A great deal of evidence indicates that such a implicit bias plays a crucial role for the generalization abilities in many modern machine learning models~\cite{zhang2016understanding,zhu2018anisotropic,wilson2017marginal,soudry2018implicit}. 
However, the implicit regularization is often a fixed effect and lacks the flexibility to be adjusted.
To fully utilize it, we need a thorough understanding about the mechanism of the implicit regularization.

Among all the efforts spent on understanding and utilizing the implicit regularization, the work on bridging iterate averaging with explicit regularization~\cite{neu2018iterate} is extraordinarily appealing. 
In particular,~\citet{neu2018iterate} show that for linear regression, one can achieve $\ell_2$-regularization effect \emph{for free} by simply taking geometrical averaging over the optimization path generated by stochastic gradient descent (SGD), which costs little additional computation.
More interestingly, the regularization is \emph{adjustable}, i.e., the solution biased by the regularizer in arbitrary strength can be obtained by iterate averaging using the corresponding weighting scheme.
In a nutshell, this regularization approach has advantages over both the implicit regularization methods for being adjustable, and the explicit regularization methods for being cheap to tune.

Nevertheless, \citet{neu2018iterate} only provide a method and its analysis for linear regression optimized by SGD.
However linear regression itself is a rather restricted optimization objective. A nature question arises:

\emph{
Can we obtain ``free'' and ``adjustable'' regularization for broader objective functions and optimization methods?
}

In this work, we answer this question positively from the following aspects:
\begin{enumerate}
\item For linear regression, we analyze the regularization effects of averaging the optimization paths of SGD as well as preconditioned SGD, with adaptive learning rates.
The averaged solutions achieve effects of $\ell_2$-regularization and generalized $\ell_2$-regularization respectively, in an adjustable manner. 
Similar results hold for kernel ridge regression as well.

\item We show that for Nesterov's accelerated stochastic gradient descent, the iterate averaged solution can also realize $\ell_2$-regularization effect by a modified averaging scheme. This resolves an open question raised by~\citet{neu2018iterate}.
    
\item Beside linear regression, we study the regularization effects of iterate averaging for strongly convex and smooth loss functions, hence establishing a provable approach for obtaining nearly \emph{free} and \emph{adjustable} regularization for a broad class of functions.

\item Empirical studies on both synthetic and real datasets verify our theory.
Moreover, we test iterate averaging with modern \emph{deep neural networks} on CIFAR-10 and CIFAR-100 datasets, and the proposed approaches \emph{still} obtain effective and adjustable regularization effects with little additional computation, demonstrating the broad applicability of our methods.
\end{enumerate}

Our analysis is motivated from continuous approximation based on differential equations.
When the learning rate tends to zero, the discrete algorithmic iterates tends to be the continuous path of an ordinary differential equation (ODE), on which we can establish a continuous version of our theory.
We then discretize the ODE and generalize the theory to that of finite step size.
This technique is of independent interests since it can be applied to analyze other comprehensive optimization problems as well~\cite{su2014differential,hu2017diffusion,li2017stochastic,yang2018physical,shi2019acceleration}.
Our results, in addition to the linear regression result in~\cite{neu2018iterate}, illustrate the promising application of iterate averaging to obtain \emph{adjustable} regularization \emph{for free}.

\section{Preliminaries}
Let $\{(x_i,y_i)\in\Rbb^{d\times 1}\}_{i=1}^n$ be the training data and $w\in\Rbb^d$ be the parameters to be optimized.
The goal is to minimize a lower bounded loss function $L(w)$
\begin{align}\label{eq:loss}\tag{$\Pcal_1$}
\min_w L(w):=\frac{1}{n}\sum_{i=1}^n\ell(x_i,y_i,w).
\end{align}
One important example is linear regression under the square loss where $L(w)=\frac{1}{2n}\sum_{i=1}^n\norm{w^\top x_i-y_i}_2^2$.
The optimization problem often involves an explicit regularization term
\begin{equation}\label{eq:reg-loss}\tag{$\Pcal_2$}
\min_{\hat{w}} L(\hat{w}) + \lambda R(\hat{w}),
\end{equation}
where $R(\hat{w})$ is a regularizer and $\lambda$ is the associated hyperparameter.
For example, the $\ell_2$-regularizer is $R(\hat{w})=\frac{1}{2}\norm{\hat{w}}_2^2$.
Given an iterative algorithm, e.g., SGD, an optimization path is generated by running the algorithm.
With a little abuse of notations, we use $\{w_k\}_{k=0}^{\infty}$ and $\{\hat{w}_k\}_{k=0}^{\infty}$ to represent the optimization paths for the unregularized problem~\eqref{eq:loss} and the regularized problem~\eqref{eq:reg-loss}, respectively.
Sometimes we write $\hat{w}_k$ with a script as $\hat{w}_{k,\lambda}$ to emphasize its dependence on the hyperparameter $\lambda$.
We use $\eta_k$ and $\gamma_k$ to denote the learning rates for training the unregularized and regularized objectives respectively.
For simplicity we always initialize the iterative algorithms from zero, i.e., $w_0=\hat{w}_0=0$.

\paragraph{Iterate averaging}
The core idea in this work is a technique called \emph{iterate averaging}.
Given a series of parameters $\{w_k\}_{k=0}^{\infty}$,
a \emph{weighting scheme} $\{p_k\}_{k=0}^{\infty}$ is defined as a probability distribution associated to the series, i.e.,
$p_k\ge 0,\ \sum_{k=0}^\infty p_k=1$. Its accumulation is denoted as $P_k = \sum_{i=0}^k p_i$, where $\lim_{k\to\infty}P_k = 1$.
Since a weighting scheme and its accumulation identifies each other by $p_k = P_k - P_{k-1}$ for $k\ge 1$, we also call $\{P_k\}_{k=0}^{\infty}$ a weighting scheme.
Then the iterate averaged parameters are 
\begin{equation*}
    \tilde{w}_k = {P_k}^{-1}{\textstyle\sum}_{i=0}^k p_i w_i,\quad k\ge 0.
\end{equation*}
Various kinds of averaging schemes (for the SGD optimization path) have been studied before. 
Theoretically, arithmetic averaging is shown to bring better convergence~\cite{bach2013non,lakshminarayanan2018linear}; tail-averaging is analyzed by~\citet{jain2018parallelizing}; and~\citet{neu2018iterate} discuss geometrically averaging and its regularization effect for SGD and linear regression.
Empirically, arithmetic averaging is also shown to be helpful for modern deep neural networks~\cite{izmailov2018averaging,zhang2019lookahead,granziol2020iterate}.
Inspired by \citet{neu2018iterate}, in this work we explore in depth the regularization effect induced by iterate averaging for various kinds of optimization algorithms and loss functions.

\paragraph{Stochastic gradient descent}
The optimization problem~\eqref{eq:loss} is often solved by stochastic gradient descent (SGD): at every iteration, a mini-batch is sampled uniformly at random, and then the parameters are updated according to the gradient of the loss estimated using the mini-batch.
For simplicity we let the batch size be $1$.
Then with learning rate $\eta_k>0$, SGD takes the following update:
\begin{equation}\label{eq:sgd-general}
    w_{k+1} = w_k - \eta_k \grad \ell(x_k, y_k, w_k).
\end{equation}
Similarly, for the regularized problem~\eqref{eq:reg-loss}, with learning rate $\gamma_k>0$, SGD takes update:
\begin{equation}\label{eq:sgd-general-regu}
    \hat{w}_{k+1} = \hat{w}_{k} - \gamma_k \left(\grad \ell(x_k,y_k,\hat{w}_{k})+\lambda \grad R(\hat{w}_k)\right).
\end{equation}
For linear regression problem and fixed learning rates, \citet{neu2018iterate} discuss the geometrically averaging over the SGD iterates~\eqref{eq:sgd-general}.
They show that by doing so one obtains the solution of the $\ell_2$-regularized problem~\eqref{eq:reg-loss} where $R(\hat{w}) = \half \norm{\hat{w}}_2^2$ for arbitrary hyperparameter $\lambda$.
In this work, we analyze a much broader class of algorithms and functions. In particular, we establish adjustable $\ell_2$-regularization effect for 
(i) SGD with adaptive learning rate, 
(ii) kernel ridge regression \cite{mohri2018foundations}, 
and (iii) general strongly convex and smooth loss functions.

\paragraph{Preconditioned stochastic gradient descent}
We also study iterate averaging for preconditioned stochastic gradient descent (PSGD).
Given a positive definite matrix $Q$ as the preconditioning matrix and $\eta_k$ as the learning rate, the PSGD takes following update to optimize problem~\eqref{eq:loss}:
\begin{equation}\label{eq:psgd-general}
    w_{k+1} = w_k - \eta_k Q^{-1} \grad \ell(x_k, y_k, w_k),
\end{equation}
Similarly, the regularized problem~\eqref{eq:reg-loss} can be solved by PSGD with learning rate $\gamma_k>0$ as:
\begin{equation}\label{eq:psgd-general-regu}
    \hat{w}_{k+1} = \hat{w}_{k} - \gamma_k Q^{-1} \left(\grad \ell(x_k,y_k,\hat{w}_{k})+\lambda \grad R(\hat{w}_k)\right).
\end{equation}
We remark that PSGD unifies several important algorithms as natural gradient descent and Newton's method at special cases where the curvature matrices can be replaced by constant matrices~\cite{martens2014new,dennis1996numerical,bottou2008tradeoffs}.

For linear regression problems, we will show that geometrically averaging the PSGD iterates~\eqref{eq:psgd-general} leads to a solution biased by the \emph{generalized $\ell_2$-regularizer}, i.e., the solution of problem~\eqref{eq:reg-loss} with $R(w) = \half w^\top Q w = \half \norm{w}_Q^2$.
The obtained regularization is adjustable, too.

\paragraph{Nesterov's accelerated stochastic gradient descent}
In problem~\eqref{eq:loss}, suppose the loss function $L(w)$ is $\alpha$-strongly convex.
Let $\eta > 0$ be the learning rate and $\tau = \frac{1-\sqrt{\eta\alpha}}{1+\sqrt{\eta\alpha}}$, then the Nesterov's accelerated stochastic gradient descent (NSGD) takes update~\cite{nesterov1983method,su2014differential,yang2018physical}:
\begin{equation}\label{eq:nsgd-general}
\begin{aligned}
    w_{k+1} &= v_k - \eta \grad \ell(x_k,y_k,v_k),\\
    v_k &= w_k + \tau (w_k-w_{k-1}).
    \end{aligned}
\end{equation}
Now we consider the regularized problem~\eqref{eq:reg-loss} with the $\ell_2$-regularizer, $R(\hat{w}) = \frac{1}{2}\norm{\hat{w}}_2^2$.
The objective function then becomes $(\alpha+\lambda)$-strongly convex. 
Let $\gamma > 0$ be the learning rate and $\hat{\tau} = \frac{1-\sqrt{\gamma(\alpha+\lambda)}}{1+\sqrt{\gamma(\alpha+\lambda)}}$, then the NSGD takes update:
\begin{equation}\label{eq:nsgd-general-regu}
\begin{aligned}
    \hat{w}_{k+1} &= \hat{v}_k - \gamma \left(\grad \ell(x_k,y_k,\hat{v}_k)+\lambda \hat{v}_k\right),\\
    \hat{v}_k &= \hat{w}_k + \hat{\tau} (\hat{w}_k-\hat{w}_{k-1}).
\end{aligned}
\end{equation}
It is proposed as an open question by~\citet{neu2018iterate} whether or not adjustable regularization can be obtained by averaging the NSGD optimization path.
Our work offers an affirmative answer by showing that for linear regression, one can perform iterate averaging over the NSGD path to obtain the $\ell_2$-regularized solution as well.

\section{The free and adjustable regularization induced by iterate averaging}
In this section, we show that adjustable regularization effects can be obtained for ``free'' via \emph{iterate averaging} for:
(i) different SGD schemes, e.g., linear regression or kernel ridge regression with adaptive learning rates;
(ii) PSGD;
(iii) NSGD;
(iv) arbitrary strongly convex and smooth loss functions.
Not limited to SGD with fixed learning rate and linear regression, our results manifest the broader potential of employing iterate averaging to obtain regularization that can be tuned with little computation overhead.

Our analysis is motivated from continuous differential equations, which is postponed to Section~\ref{sec:continuous-analysis} of Supplementary Materials due to space limitation. In the following we present our results in discrete cases.

\subsection{The effect of an averaged SGD path}

\begin{figure}
\centering
\includegraphics[width=0.85\linewidth]{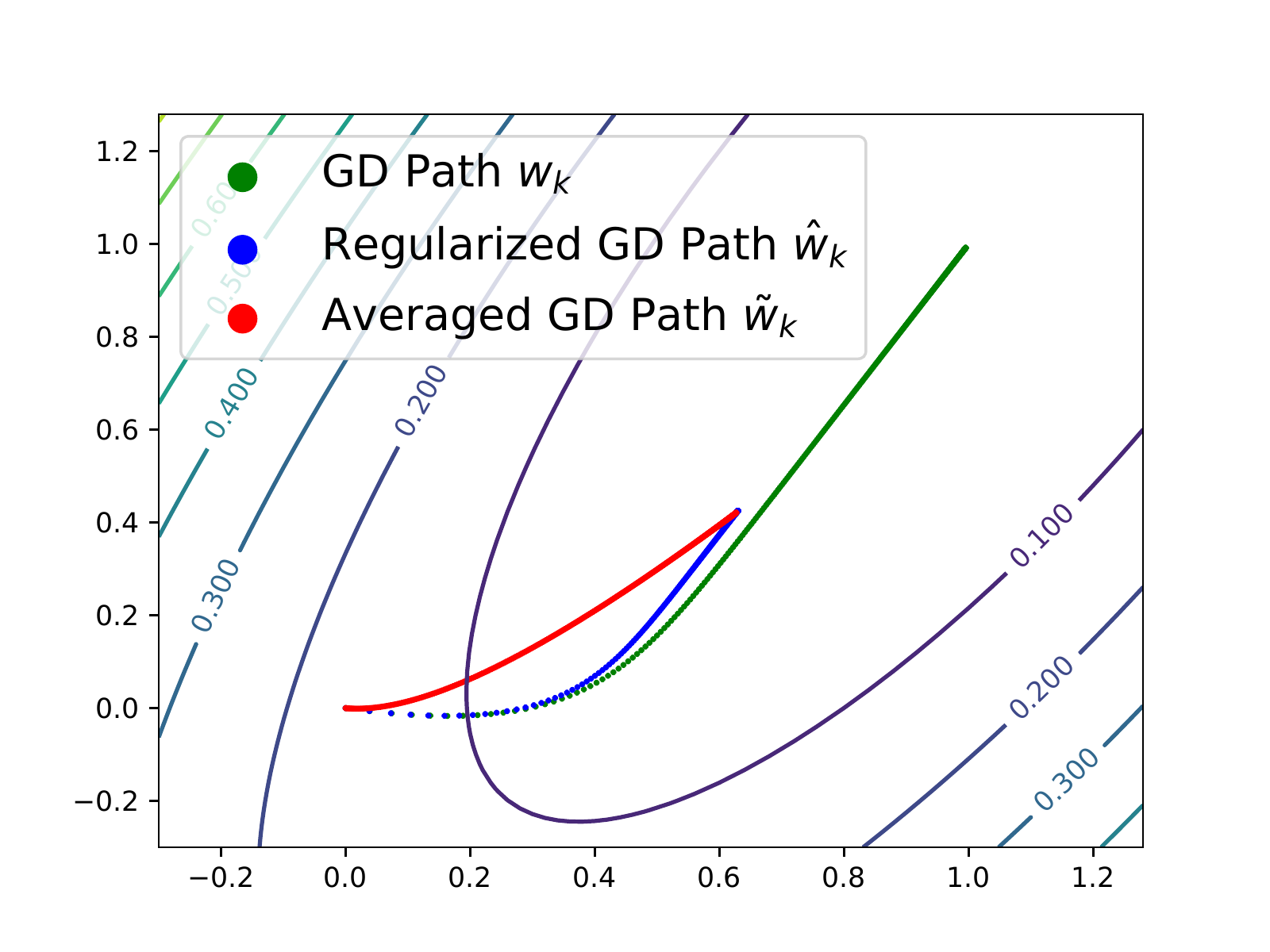}
\caption{\small A 2-D demonstration of the effect of an averaged SGD path (Theorem~\ref{thm:sgd-linear}).
{Green dots}: the vanilla GD path $w_k$;
{blue dots}: the regularized GD path $\hat{w}_k$;
{red dots}: the averaged GD path $\tilde{w}_k$.
The red dots converge to the blue ones.}
\label{fig:2dim-gd}
\end{figure}

We first introduce a generalized averaging scheme for the SGD algorithm. 
Unlike the method in \cite{neu2018iterate}, our approach works even with adaptive learning rates.
Specifically, given a learning rate schedule and a regularization parameter $\lambda$, we compute a weighting scheme for averaging a stored SGD path.
Then the averaged solution converges to the regularized solution with hyperparameter $\lambda$.
Theorem~\ref{thm:sgd-linear} formally justifies our method.

\begin{thm}[The effect of an averaged SGD path]\label{thm:sgd-linear}
Consider loss function $L(w)=\frac{1}{2n}\sum_{i=1}^n\norm{w^\top x_i-y_i}_2^2$, and regularizer $R(w) = \half \norm{w}_2^2$. 
Let $\alpha$ and $\beta$ be such that $L(w)$ is $\alpha$-strongly convex\footnotemark and $\beta$-smooth.
Let $\{w_k\}_{k=0}^{\infty}$ and $\{\hat{w}_k\}_{k=0}^{\infty}$ be the SGD paths for the vanilla loss function $L(w)$ with learning rate $\eta_k$, and the regularized loss function $ L(\hat{w})+ \lambda R(\hat{w})$ with learning rate $\gamma_k$, respectively.
Suppose
$1 - \lambda \gamma_k = {\gamma_k}/{\eta_k}$, $\eta_k \in (\eta, 1/\beta)$, $\eta > 0$ and $\gamma := {\eta}/{(1+\lambda\eta)}$.
Let 
\begin{equation*}
P_k := {\textstyle\sum}_{i=0}^k p_i = 1 - \Pi_{i=0}^k ({\gamma_i}/{\eta_i}).
\end{equation*}
Then for $\tilde{w}_k=P_k^{-1}\sum_{i=0}^k p_i w_i$ we have

1.
$P_k \cdot\Ebb[\tilde{w}_k] = \Ebb[\hat{w}_k] - (1-P_k)\cdot \Ebb[w_k]$.

2.
Both $\Ebb[w_k]$ and $\Ebb[\hat{w}_k]$ converge. Moreover, we have
$\norm{\Ebb[\hat{w}_k] - \Ebb[\tilde{w}_k]}_2 \le \Ocal((1-\lambda\gamma)^k)$.

3.
If the gradient noise $\epsilon_k = \grad \ell(x_k,y_k,w) - \grad L(w)$ has uniformly bounded variance $\Ebb[\norm{\epsilon_k}_2^2] \le \sigma^2$, then for $k$ large enough, with probability at least $1-\delta$ we have\footnotemark
\begin{equation*}
    \norm{P_k \tilde{w}_k - P_k \Ebb[\tilde{w}_k]}_2 \le \epsilon,
\end{equation*}
where
$\epsilon = \frac{\sigma}{\gamma(\lambda+\alpha)(\lambda+\beta)^2}$ $ \cdot \sqrt{\frac{\lambda}{\delta\gamma(2-\lambda\gamma)}}$.

\end{thm}
\footnotetext[1]{The strong convexity assumption does not limit the application of our method. 
% Note we aim to collect a single optimization path and convert it to solutions with adjustable regularization effect.
For a convex but not strongly convex loss $L(w)$, we can instead collect an optimization path of $L(w) + \lambda_0 \norm{w}^2_2$ for some small $\lambda_0$, which is then strongly convex, and then we apply Theorem~\ref{thm:sgd-linear} to obtain the regularized solutions for a different $\lambda$.
Similar arguments apply to the theorems afterwards as well.}
\footnotetext[2]{In this high probability result, the confidence parameter $\delta$ appears in a polynomial order, $\frac{1}{\sqrt{\delta}}$. However this is only due to the assumption of bounded variance of the noise and an application of Chebyshev's inequality. It is straightforward to obtain a logarithm dependence on $\delta$ by assuming the sub-Gaussianity of the noise and applying Hoeffding's inequality.
Similar arguments apply to the theorems afterwards as well.}

The proof is left in Supplementary Materials, Section~\ref{sec:proof-sgd-linear}.
A 2-D illustration for Theorem~\ref{thm:sgd-linear} is presented in Figure~\ref{fig:2dim-gd}.

Theorem~\ref{thm:sgd-linear} guarantees the method of obtaining adjustable $\ell_2$-regularization for free via iterate averaging.
Specifically, we first collect an SGD path $\set{w_k}_{k=0}^{\infty}$ for $L(w)$ under a learning rate schedule $\eta_k$ (it can be chosen in a broad range);
then for a regularization parameter $\lambda$, we compute an averaging scheme $\set{p_k}_{k=0}^{\infty}$ that converts the collected SGD path to the regularized solution, $\hat{w}_{\infty}$. Note that the learning rate schedule $\gamma_k$ is only for analysis and does not need to be known.

Specifically, when the learning rates are constants, i.e., $\eta_i=\eta$ and $\gamma_i=\gamma$, the first two conclusions in Theorem~\ref{thm:sgd-linear} recover the Proposition~1 and Proposition~2 in~\cite{neu2018iterate}.
Besides, the third claim in Theorem~\ref{thm:sgd-linear} characterizes the deviation of the averaged solution, which relies on the models, learning rates, and the regularization parameter, etc. 
And empirical studies in Section~\ref{sec:exp-mnist} do suggest that such a deviation is sufficiently small that it does not affect the induced regularization effect.

\begin{rmk}
We emphasize that the method of \citet{neu2018iterate} only applies to SGD with constant learning rate.
Moreover, their theory only guarantees the averaged solution has \emph{convergence in expectation}, which is not very useful since the averaged solution might not converge to the regularized solution almost surely, not even in probability (a.k.a. weak convergence) (see Section~\ref{sec:exp-mnist}).
Nevertheless, our theory carefully characterizes the deviation between the averaged solution and the regularized solution.
\end{rmk}

More interestingly, Theorem~\ref{thm:gd-kernel} shows that this method is also applicable to kernel ridge regression (in the dual space).

\begin{thm2}\label{thm:gd-kernel}
Let $K\in \Rbb^{n\times n}$ be a kernel, $K(i,j) = \phi(x_i)^\top \phi(x_j)$, where $\phi: \Rbb^d \to \Hcal$ is the kernel map.
Consider kernel ridge regression
\begin{equation*}
    \min_{\alpha\in \Rbb^n} L(\alpha,\lambda) := \half \norm{y-K\alpha}_2^2 + \frac{\lambda}{2} \alpha^\top K \alpha
\end{equation*}
where $y=(y_1,\dots,y_n)^T$ is the labels and $\alpha\in \Rbb^n$ is the dual parameter.
Let $\{\alpha_k\}_{k=0}^{\infty}$ and $\{\hat{\alpha}_k\}_{k=0}^{\infty}$ be the GD paths for the loss $L(\alpha,\lambda)$ with learning rate $\eta_k$, and the loss $L(\hat{\alpha},\hat{\lambda})$ with generalized learning rate $\gamma_k$, respectively.
Suppose $\hat{\lambda} > \lambda$,
$\gamma_k = \eta_k\left(I + (\hat{\lambda} - \lambda) \eta_k K \right)^{-1} $.
Let 
\begin{equation*}
    P_k := {\textstyle\sum}_{i=0}^k p_i = 1 - \Pi_{i=0}^k \left(\gamma_i / \eta_i\right).
\end{equation*}
Then for $\tilde{\alpha}_k = P_k^{-1}\sum_{i=0}^k p_i \alpha_i$ we have

1.
$P_k \tilde{\alpha}_k = \hat{\alpha}_k - (1-P_k) \alpha_k$.

2.
Both $\alpha_k$ and $\hat{\alpha}_k$ converge provided suitable learning rates. Moreover, we have
$\norm{\hat{\alpha}_k - \tilde{\alpha}_k}_2 \le \Ocal(C^k)$ where $C\in(0,1)$ is a constant decided by $K$, $\hat{\lambda}-\lambda$ and $\eta_k$.

\end{thm2}

\subsection{The effect of an averaged PSGD path}

\begin{figure}
\centering
\includegraphics[width=0.85\linewidth]{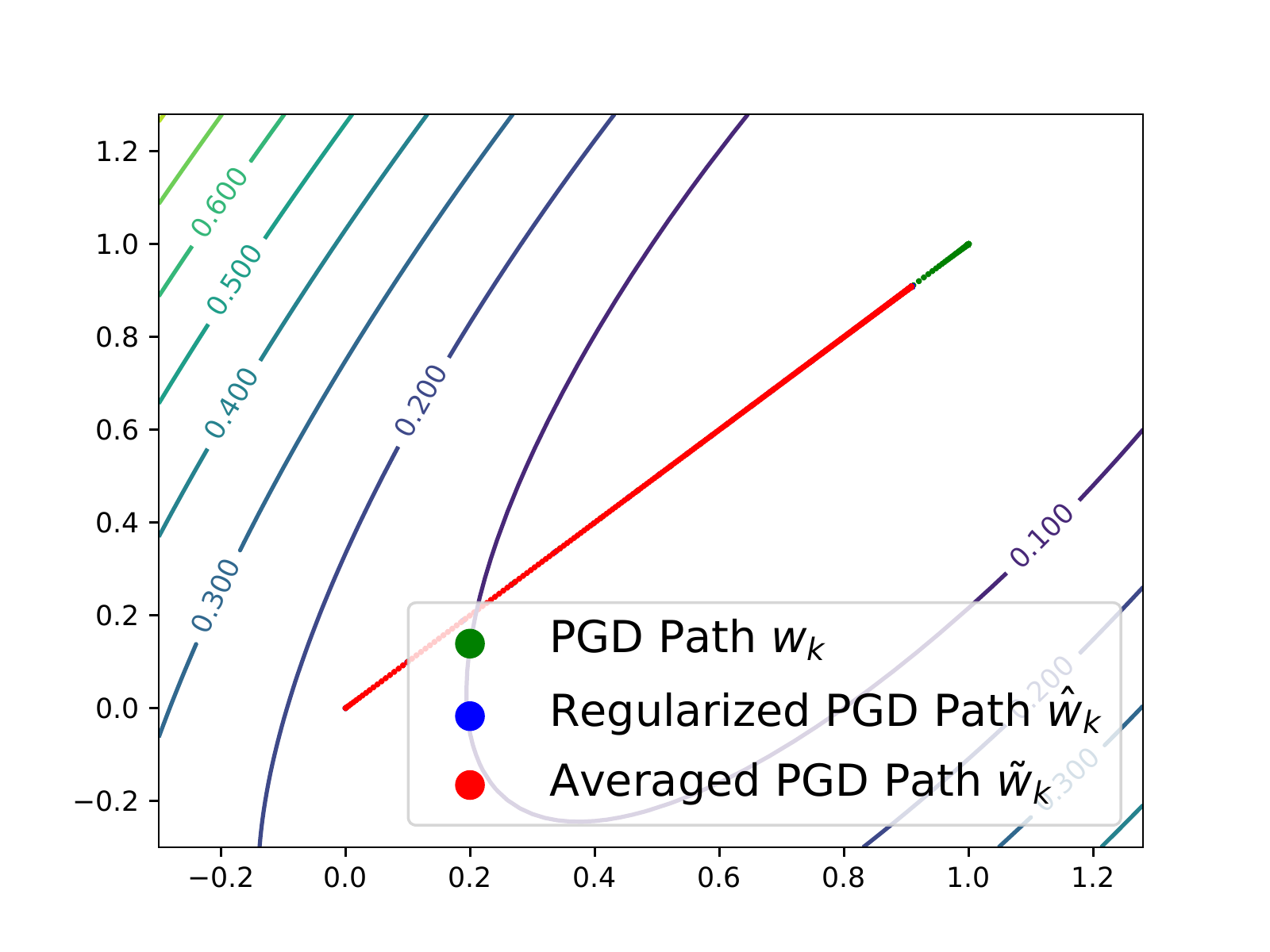}
\caption{\small A 2-D demonstration of the effect of an averaged PSGD path (Theorem~\ref{thm:psgd-linear}).
The Hessian is used as the preconditioning matrix.
{Green dots}: the vanilla PGD path $w_t$;
{blue dots}: the regularized PGD path $\hat{w}_t$;
{red dots}: the averaged PGD path $\tilde{w}_t$.
The red dots converge to the blue ones.}
\label{fig:2dim-pgd}
\end{figure}

In practice, we usually need many different regularizers.
And one important class of them is the \emph{generalized $\ell_2$-regularizers}, i.e.,
$R(w):=\half w^\top Q w$ for some positive definite matrix $Q$. 
But it is painful to adjust its regularization parameter $\lambda$ by re-training the model.
Luckily, we show that the solution biased by such a regularizer can also be obtained for ``free'' by averaging the optimization path of PSGD.
Our result is formally presented in the next theorem.

\begin{thm}[The effect of an averaged PSGD path]\label{thm:psgd-linear}
Consider loss function $L(w)=\frac{1}{2n}\sum_{i=1}^n\norm{w^\top x_i-y_i}_2^2$, and regularizer $R(w) = \half w^\top Q w$, where $Q$ is a positive definite matrix.
Let $\alpha$ and $\beta$ be such that $\alpha Q \preceq \Sigma = n^{-1}\sum_{i=1}^n x_i x_i^\top \preceq \beta Q$.
With $Q$ as the preconditioning matrix, let $\{w_k\}_{k=0}^{\infty}$ and $\{\hat{w}_k\}_{k=0}^{\infty}$ be the PSGD paths for the vanilla loss function $L(w)$ with learning rate $\eta_k$, and the regularized loss function $L(\hat{w})+ \lambda R(\hat{w})$ with learning rate $\gamma_k$, respectively.
Suppose
$1-\lambda \gamma_k = {\gamma_k}/{\eta_k}$, $\eta_k\in(\eta, 1/\beta)$, $\eta > 0$ and $\gamma := {\eta}/{(1+\lambda\eta)}$.
Let 
\begin{equation*}
    P_k := {\textstyle\sum}_{i=0}^k p_i = 1 - \Pi_{i=0}^k ({\gamma_i}/{\eta_i}).
\end{equation*}
Then for $\tilde{w}_k=P_k^{-1}\sum_{i=0}^k p_i w_i$ we have

1.
$P_k \cdot \Ebb[\tilde{w}_k] = \Ebb[\hat{w}_k] - (1-P_k)\cdot \Ebb[w_k]$.

2.
Both $\Ebb[w_k]$ and $\Ebb[\hat{w}_k]$ converge. Moreover, we have
$\norm{\Ebb[\hat{w}_k] - \Ebb[\tilde{w}_k]}_2 \le \Ocal((1-\lambda\gamma)^k)$.

3. 
If the noise $\epsilon_k = Q^{-1}(\grad \ell(x_k,y_k,w) -\grad L(w))$ has uniform bounded variance $\Ebb[\norm{\epsilon_k}_2^2] \le \sigma^2$, then for $k$ large enough, with probability at least $1-\delta$ we have
\begin{equation*}
    \norm{P_k \tilde{w}_k - P_k \Ebb[\tilde{w}_k]}_2 \le \epsilon,
\end{equation*}
where $\epsilon = \frac{\sigma  \norm{Q}_2 }{\gamma(\lambda+\alpha)(\lambda+\beta)^2}\sqrt{\frac{\lambda}{\delta\gamma(2-\lambda\gamma)}}$.
\end{thm}
The proof is left in Supplementary Materials, Section~\ref{sec:proof-psgd-linear}.
A 2-D illustration for Theorem~\ref{thm:psgd-linear} is presented in Figure~\ref{fig:2dim-pgd}.

The importance of Theorem~\ref{thm:psgd-linear} is two-folds.
On the one hand, averaging the PSGD path has an effect as the generalized $\ell_2$-regularizer.
And as before, this induced regularization is both adjustable and costless.
The considered PSGD algorithm applies to natural gradient descent and Newton's method in certain circumstances where the curvature matrices can be replaced by constant matrices~\cite{martens2014new,dennis1996numerical,bottou2008tradeoffs}.
On the other hand, to obtain a desired type of generalized $\ell_2$-regularization effect, we should store and average a PSGD path with the corresponding preconditioning matrix as indicated in Theorem~\ref{thm:psgd-linear}, instead of using a SGD path.

\subsection{The effect of an averaged NSGD path}

\begin{figure}
\centering
\includegraphics[width=0.85\linewidth]{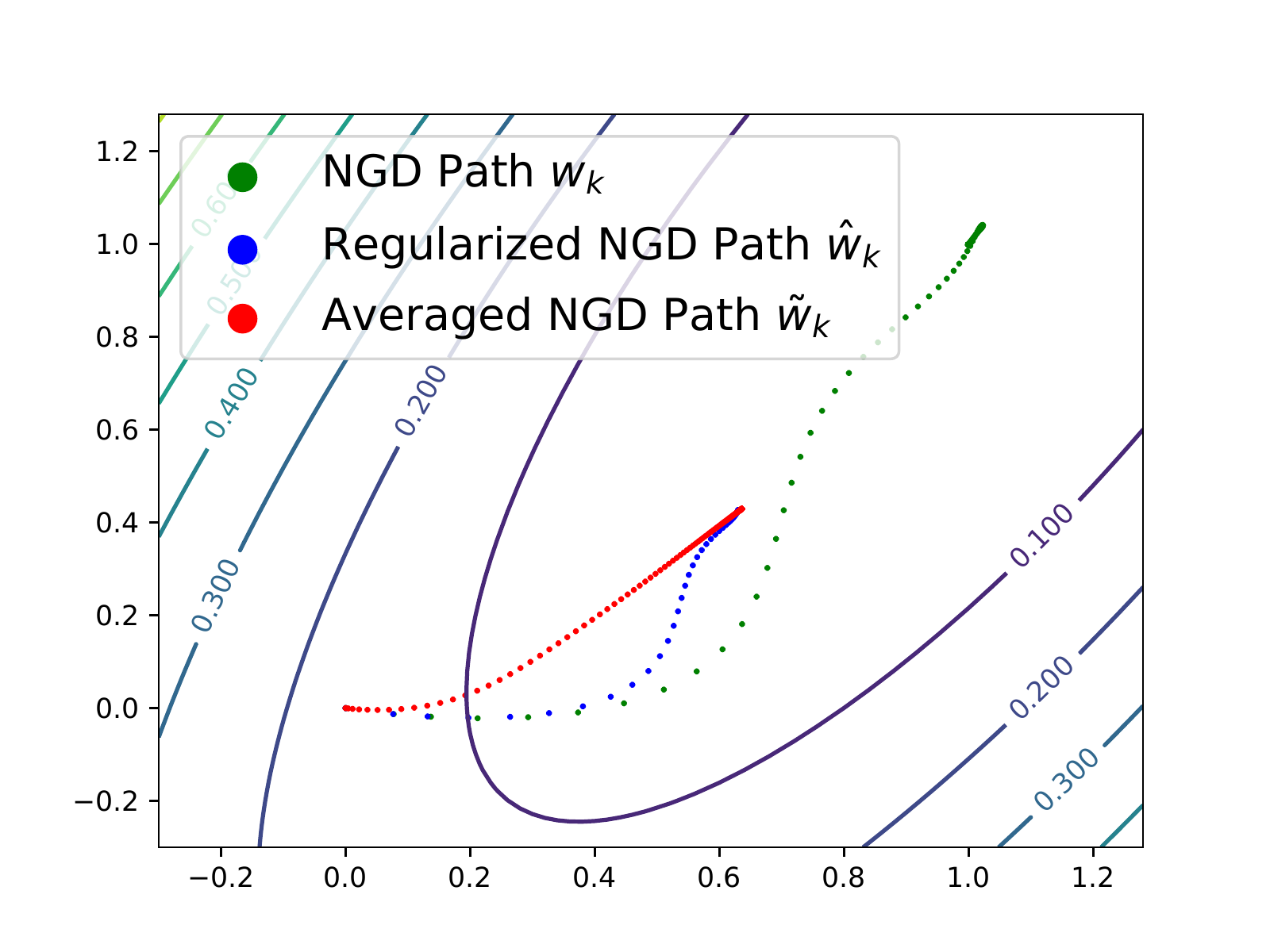}
\caption{\small A 2-D demonstration of the effect of an averaged NSGD path (Theorem~\ref{thm:nsgd-linear}).
{Green dots}: the vanilla NGD path $w_t$;
{blue dots}: the regularized NGD path $\hat{w}_t$;
{red dots}: the averaged NGD path $\tilde{w}_t$.
The red dots converge to the blue ones.}
\label{fig:2dim-ngd}
\end{figure}

In this part, we show how to obtain adjustable regularization effect by applying averaging schemes on the NSGD path.

\begin{thm}[The effect of an averaged NSGD path]\label{thm:nsgd-linear}
Consider loss function $L(w) = \frac{1}{2n}\sum_{i=1}^n\norm{w^\top x_i-y_i}_2^2$, and regularizer $R(w) = \half \norm{w}_2^2$. 
Let $\alpha$ and $\beta$ be such that $L(w)$ is $\alpha$-strongly convex and $\beta$-smooth.
Let $\{w_k\}_{k=0}^{\infty}$ and $\{\hat{w}_k\}_{k=0}^{\infty}$ be the NSGD paths for the vanilla loss function $L(w)$ with learning rate $\eta$, and the regularized loss function $L(\hat{w})+ \lambda R(\hat{w})$ with learning rate $\gamma$, respectively.
Suppose
$1-\lambda \gamma = {\gamma}/{\eta}$,
$\eta\in(0,1/\beta)$.
Let 
\begin{equation*}
    P_k := \sum_{i=0}^k p_k = 1 - \frac{\gamma}{\eta}\left( \frac{1-\sqrt{\gamma(\alpha+\lambda)}}{1-\sqrt{\eta\alpha}} \right)^{k-1}.
\end{equation*}
Then for $\tilde{w}_k=P_k^{-1}\sum_{i=0}^k p_i w_i$ we have

1. 
$P_k \cdot \Ebb[\tilde{w}_k] = \Ebb[\hat{w}_k] - (1-P_k)\cdot \Ebb[w_k]$.

2. 
$\Ebb[w_k]$ and $\Ebb[\hat{w}_k]$ converge. And 
$\norm{\Ebb[\hat{w}_k] - \Ebb[\tilde{w}_k]}_2 \le \Ocal\left(C^k\right)$, where $C = \frac{1-\sqrt{\gamma(\alpha+\lambda)}}{1-\sqrt{\eta\alpha}}\in (0,1)$.

3.
If the gradient noise $\epsilon_k = \grad \ell(x_k,y_k,w) - \grad L(w)$ has uniformly bounded variance $\Ebb[\norm{\epsilon_k}_2^2] \le \sigma^2$, then for $k$ large enough, with probability at least $1-\delta$ we have
\begin{equation*}
    \norm{P_k \tilde{w}_k - P_k \Ebb[\tilde{w}_k]}_2 \le \epsilon,
\end{equation*}
where $\epsilon$ depends on $\sigma,\alpha,\beta,\eta,\gamma$.
% $\epsilon = \frac{\sigma \left( \sqrt{\gamma(\alpha+\lambda)} - \sqrt{\eta\alpha} \right) \left(1-\sqrt{\eta\alpha}\right)\sqrt{\gamma (1-\eta\alpha)}}{\sqrt{\delta \eta(\lambda_{\min} -\alpha) (\alpha+\lambda) \left(1 - \left(1-\sqrt{\eta\alpha}\right)^2 \left(1-\sqrt{\gamma(\alpha+\lambda)}\right)^2\right)}}$.

\end{thm}
The proof and the exact value of $\epsilon$ are given in Supplementary Materials, Section~\ref{sec:proof-nsgd-linear}.
A 2-D illustration for Theorem~\ref{thm:nsgd-linear} is presented in Figure~\ref{fig:2dim-ngd}.

Theorem~\ref{thm:nsgd-linear} affirmatively answers an open question raised by \citet{neu2018iterate}: there exists an averaging scheme for NSGD to achieve $\ell_2$-regularization in arbitrary strength.
In addition to the results for averaging SGD, Theorem~\ref{thm:nsgd-linear} provides us wider choices of applicable optimizers for obtaining adjustable $\ell_2$-regularization effect by iterate averaging.

\subsection{The effect of an averaged GD path for strongly convex and smooth loss functions}

In this section, we show that the iterate averaging methods work for not only simple optimization objectives like least square, but also a much broader set of loss functions.
In fact, we show that any strongly convex and smooth function admits an iterate averaging scheme, which brings $\ell_2$-regularization effect in a tunable manner.
More formally, in the problems~\eqref{eq:loss} and~\eqref{eq:reg-loss}, let $L(w)$ be $\alpha$-strongly convex and $\beta$-smooth, and $R(w) := \half \norm{w}_2^2$ be the $\ell_2$-regularizer.
For the sake of representation, we focus on gradient descent (GD) with constant learning rate applied on the loss functions.
Similar arguments can also be applied for SGD, PSGD and NSGD.
The GD takes update
\begin{equation*}
\begin{aligned}
    & w_{k+1} = w_k - \eta\grad L(w_k),\\ 
    & \hat{w}_{k+1, \lambda} = \hat{w}_{k,\lambda} - \gamma(\grad L(\hat{w}_{k,\lambda})+\lambda\hat{w}_{k,\lambda}),
\end{aligned}
\end{equation*}
for optimizing problems~\eqref{eq:loss} and~\eqref{eq:reg-loss}, respectively.
Let $b = -\grad L(w_0) = -\grad L(0)$.
Let us denote two iterations
\begin{equation*}
    u_{k+1}-u_k = -\eta(\alpha u_k-b),\quad
    v_{k+1}-v_k = -\eta(\beta v_k-b),
\end{equation*}
where $u_0=v_0=0$.
Consider an averaging scheme $P_k = \sum_{i=0}^k p_i = 1 - \left({\gamma}/{\eta}\right)^{k+1}$.
Let $\tilde{u}_k = {P_k}^{-1}\sum_{i=0}^k p_i u_i$, $\tilde{v}_k = {P_k}^{-1}\sum_{i=0}^k p_i v_i$,
and $\tilde{w}_k=P_k^{-1}\sum_{i=0}^k p_i w_i$.
Then the next theorem characterizes the regularization effect of a averaged GD path for general strongly convex and smooth loss functions.

\begin{thm}[The effect of an averaged GD path for strongly convex and smooth loss functions]\label{thm:sgd-general}
Without loss of generality, assume the unique minimum $w_*$ of $L(w)$ satisfies $w_* > w_0 = 0$ entry-wisely.
Suppose
${1}/{(2\beta-\alpha)} < \eta < {1}/{\beta}$,
$0<\gamma<{\eta}/{(\eta(\beta-\alpha)+1)}$.
Then for hyperparameters
\begin{equation*}
    \lambda_1 = {1}/{\gamma}-{1}/{\eta} + \beta-\alpha,\quad
    \lambda_2 = {1}/{\gamma}-{1}/{\eta} + \alpha-\beta,
\end{equation*}
we have

1.
$\hat{w}_{k,\lambda_1} + (1-P_k)(\tilde{v}_k-v_k) \le \tilde{w}_k \le \hat{w}_{k,\lambda_2} + (1-P_k)(\tilde{u}_k-u_k)$,
where the ``$\le$'' is defined entry-wisely.

2. $u_k, \tilde{u}_k, v_k, \tilde{v}_k, \hat{w}_{k,\lambda_1}, \hat{w}_{k,\lambda_2}$ converge.
Moreover let $m = (\hat{w}_{\infty,\lambda_2}+\hat{w}_{\infty,\lambda_1})/2$, $d = (\hat{w}_{\infty,\lambda_2}-\hat{w}_{\infty,\lambda_1})/2$ and $C = \max\{(1-\gamma(\alpha+\lambda_1), (1-\gamma(\alpha+\lambda_2), \frac{\gamma}{\eta}\} \in (0,1)$, then
$\norm{\tilde{w}_k - m}_2 \le \norm{d}_2 + \Ocal(C^k)$.

\end{thm}
The proof is left in Supplementary Materials, Section~\ref{sec:proof-sgd-general} .

According to Theorem~\ref{thm:sgd-general}, for strongly convex and smooth objectives, the averaged GD path $\{\tilde{w}_k\}_{k=0}^{\infty}$ lies in the area between two regularized GD paths, $\{\hat{w}_{k,\lambda_1}\}_{k=0}^{\infty}$ and $\{\hat{w}_{k,\lambda_2}\}_{k=0}^{\infty}$.
Furthermore, $\tilde{w}_k$ converges to a hyper cube whose diagonal vertices are defined by $\hat{w}_{\infty,\lambda_1}$ and $\hat{w}_{\infty,\lambda_2}$.
In this way for this class of loss functions, averaging the GD path has an ``approximate'' $\ell_2$-regularization effect 
that is in between two $\ell_2$-regularizers with hyperparameters as $\lambda_1$ and $\lambda_2$ respectively.
In addition, $\lambda_1$ and $\lambda_2$ can be adjusted through changing the weighting scheme.
Finally, we note that $\norm{d}_2=\bigO{\beta-\alpha}$, thus when the objective is quadratic, we have $\alpha = \beta$ and $d = 0$ and thus the ``approximate'' $\ell_2$-regularization effect becomes the exact $\ell_2$-regularization by Theorem~\ref{thm:sgd-general}.

We therefore conjecture that, generally, \emph{for arbitrary loss functions and iterative optimizers, an iterate averaging scheme admits a specific yet unknown regularization effect.}
Indeed, our experiments in the next section empirically verifies such an effect by performing iterate averaging on deep neural networks, which are highly comprehensive.

\section{Experiments}
In this section we present our empirical studies.
The detailed setups are explained in Supplementary Materials, Section~\ref{sec:exp-setup}.
The code is available at \url{https://github.com/uuujf/IterAvg}.

\subsection{Two dimensional demonstration}
We first introduce a two dimensional toy example to demonstrate the regularization effect of iterate averaging.
The vanilla loss function is quadratic with a unique minimum at $(1,1)$, as shown in Figure~\ref{fig:2dim-gd}$\sim$\ref{fig:2dim-ngd}.
For the purpose of demonstration we only run deterministic algorithms with constant learning rates.
We plot the trajectories of the concerned optimizers for learning the vanilla loss function and the regularized loss function, as well as the averaged solutions.
All of the optimizers start iterations from zero.

In Figure~\ref{fig:2dim-gd}, the green and the blue dots represent the GD paths for optimizing the  vanilla/regularized loss functions respectively, while the red dots are the path of iterate averaged solutions.
We observe that the red dots do converge to the blue ones, indicating the averaged solution has the same effect of an $\ell_2$-regularizer, as suggested by Theorem~\ref{thm:sgd-linear}.
Similarly the phenomenon holds for averaging the NGD path, as indicated in Figure~\ref{fig:2dim-ngd}.
In Figure~\ref{fig:2dim-pgd}, the preconditioning matrix is set to be the Hessian.
And as predicted by Theorem~\ref{thm:psgd-linear}, the averaged solution converges to the solution biased by a generalized $\ell_2$-regularizer.

\subsection{Real data verification}\label{sec:exp-mnist}
We then present experiments on the MNIST dataset.

\begin{figure*}
\centering
\begin{tabular}{ccc}
\includegraphics[width=0.3\linewidth]{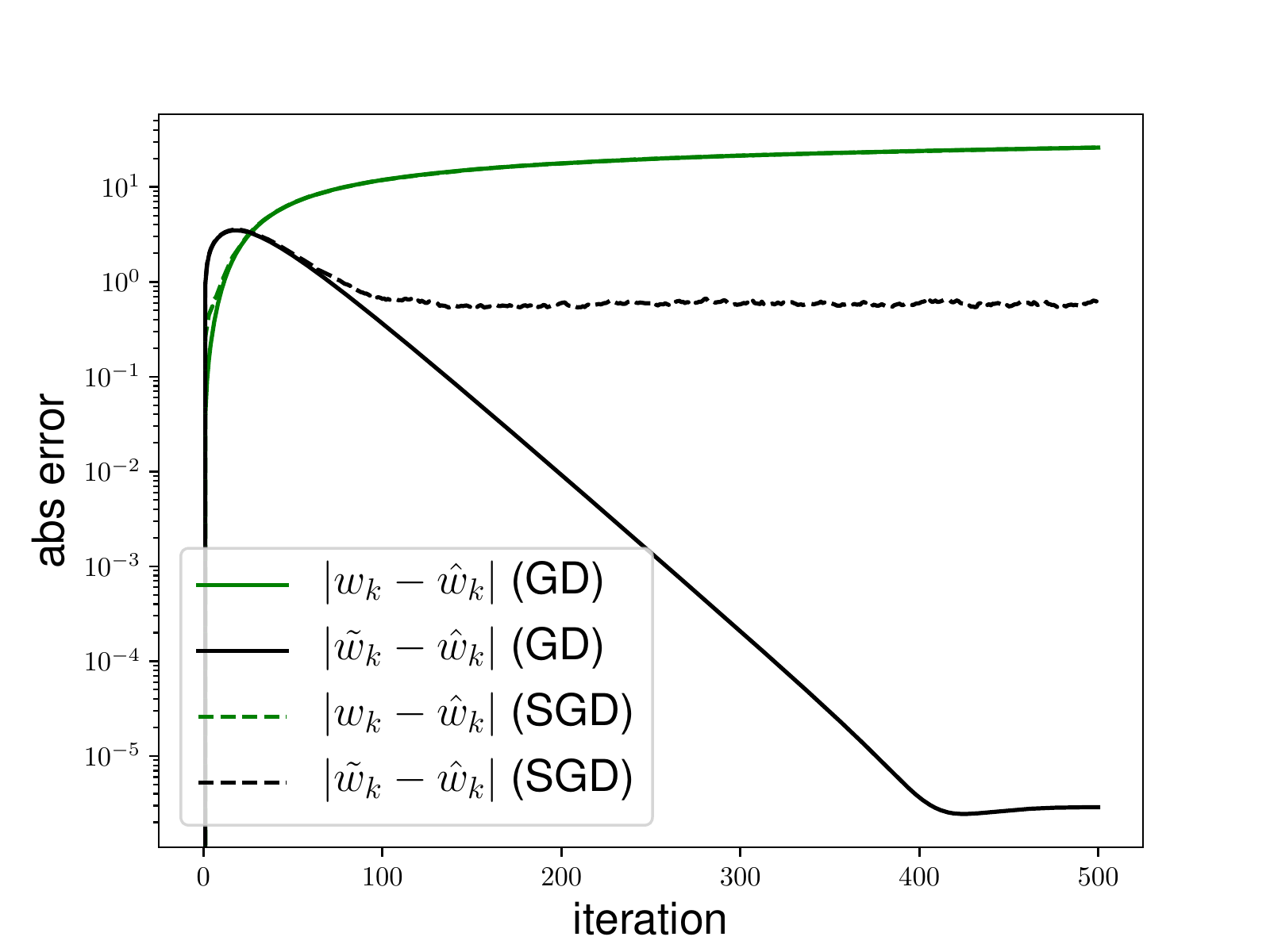} & 
\includegraphics[width=0.3\linewidth]{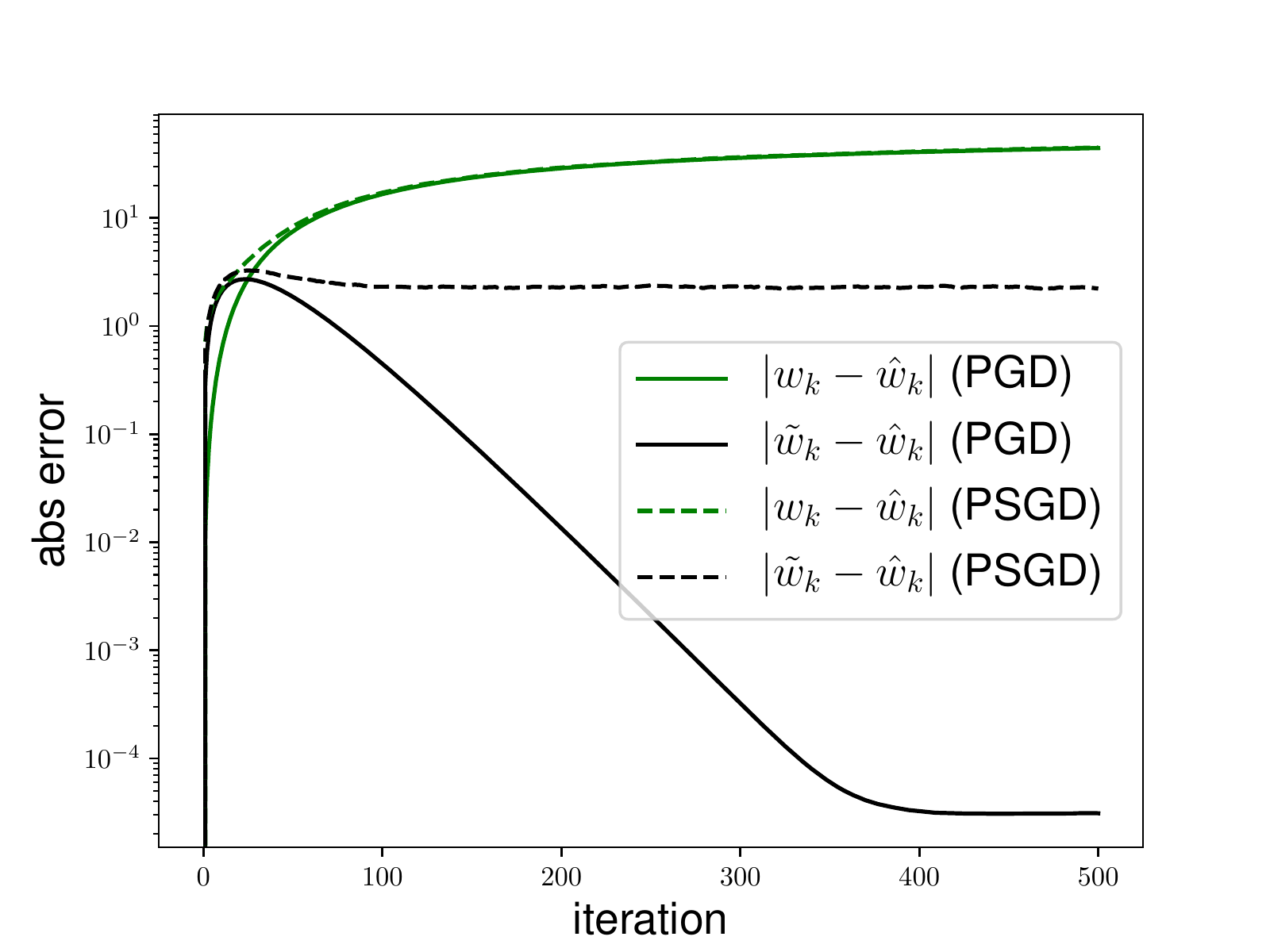} &
\includegraphics[width=0.3\linewidth]{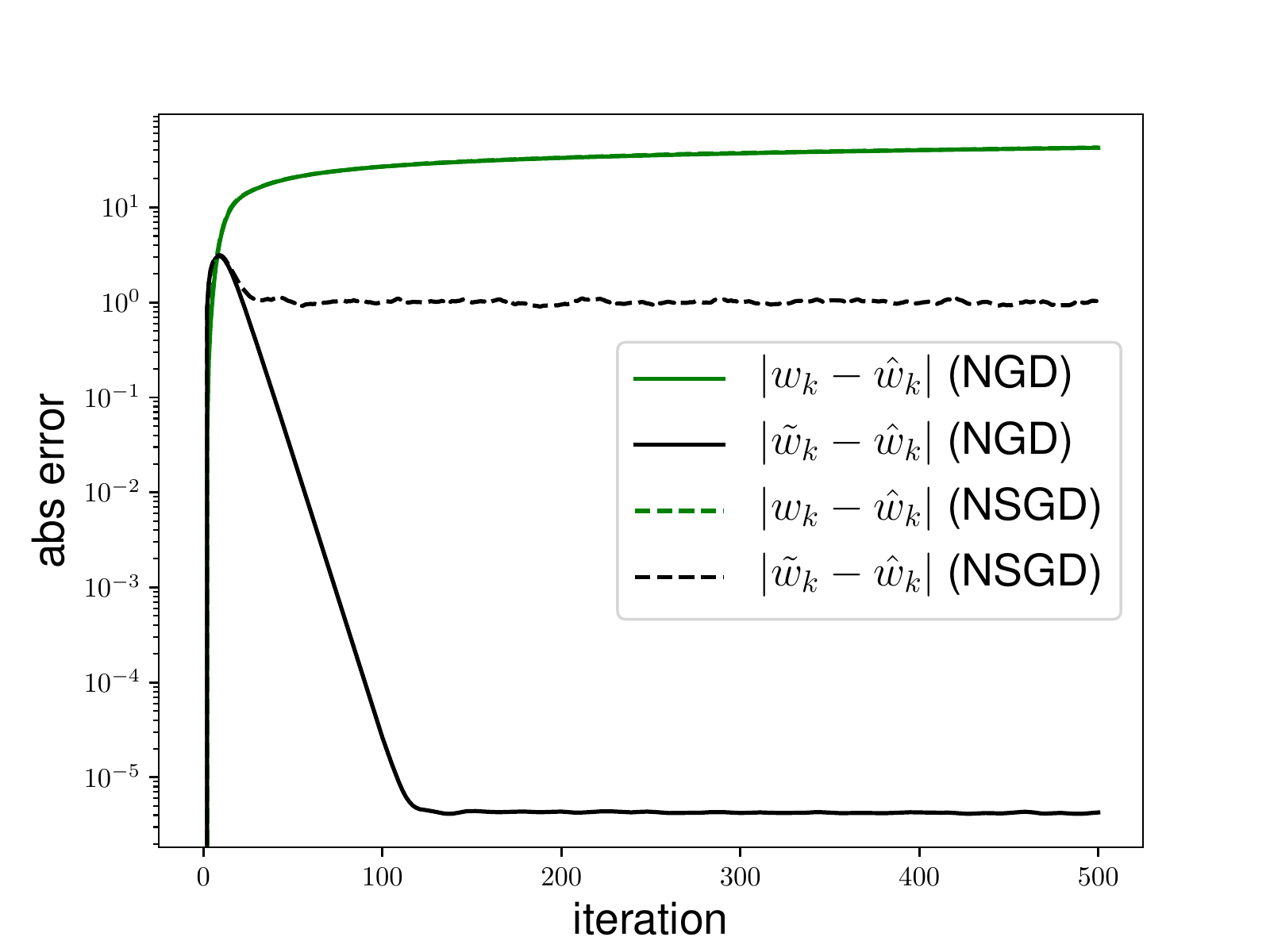} \\
(a) GD and SGD & (b) PGD and PSGD & (c) NGD and PNGD
\end{tabular}
\caption{\small
Linear regression on MNIST dataset.
X-axis: iteration; y-axis: the absolute approximation error in logarithmic scale.
Green lines represent $\norm{{w}_t - \hat{w}_t}_1$ and black lines represent $\norm{\tilde{w}_t - \hat{w}_t}_1$, where $w_t$, $\hat{w}_t$ and $\tilde{w}_t$ are the unregularized path, the regularized path and the iterate averaged path, respectively.
Solid lines and dashed lines are the results obtained by running deterministic and stochastic algorithms respectively.
For deterministic algorithms, the error between $\tilde{w}_t$ and $\hat{w}_t$ converges to zero.
For stochastic algorithms, the error between $\tilde{w}_t$ and $\hat{w}_t$ remains small.
}
\label{fig:mnist-linear-regression}
\end{figure*}

\paragraph{Linear regression}
Firstly, we study linear regression under quadratic loss functions and the regularization effects caused by averaging the optimization paths of (S)GD, P(S)GD and N(S)GD.
The learning rates are set to be constant. 
For P(S)GD, we set the preconditioning matrix as the Hessian, which is known as the Newton's method.

Our theories predict that averaging the (S)GD and N(S)GD paths leads to the solutions biased by $\ell_2$-regularizers (Theorem~\ref{thm:sgd-linear},~\ref{thm:nsgd-linear}), while averaging the P(S)GD path introduces an effect of the generalized $\ell_2$-regularization (Theorem~\ref{thm:psgd-linear}).
To verify the predictions, 
we generate the paths of the averaged solutions $\tilde{w}_k$ and the regularized solutions $\hat{w}_k$, and then compute the approximation errors between them.
The results are plotted in Figure~\ref{fig:mnist-linear-regression}.

In Figure~\ref{fig:mnist-linear-regression} (a), the solid lines clearly indicate that the averaged solution converges to the regularized solution when running GD, which also corresponds to the convergence in expectation in SGD cases, as predicted by Theorem~\ref{thm:sgd-linear}.
For SGD, however, the dashed lines in Figure~\ref{fig:mnist-linear-regression} (a) show that there is a small error between the averaged solution and the regularized solution.
The error exists since the convergence of the averaged solution does not hold in probability.
Luckily, the error would not grow large as the deviation of the averaged solution is controllable by Theorem~\ref{thm:sgd-linear}.
Hence by comparing the dashed green and black lines, we see that averaging the SGD path still leads to an effect of $\ell_2$-regularization ignoring a tolerable error.

Figure~\ref{fig:mnist-linear-regression} (b) shows the results for PGD and PSGD.
Again, averaging the PGD path causes a perfect generalized $\ell_2$-regularization effect, and there is a small gap for averaging the optimization path with noise.
These support Theorem~\ref{thm:psgd-linear}.

The results related to NGD and NSGD are shown in Figure~\ref{fig:mnist-linear-regression} (c).
Again, for the deterministic algorithm, the solid lines manifest the convergence between the averaged solution and the regularized solution, verifying our Theorem~\ref{thm:nsgd-linear}.
And the dashed lines once more suggest the stochastic algorithm causes a tolerable approximation error.

\begin{figure*}
\centering
\begin{tabular}{ccc}
\includegraphics[width=0.3\linewidth]{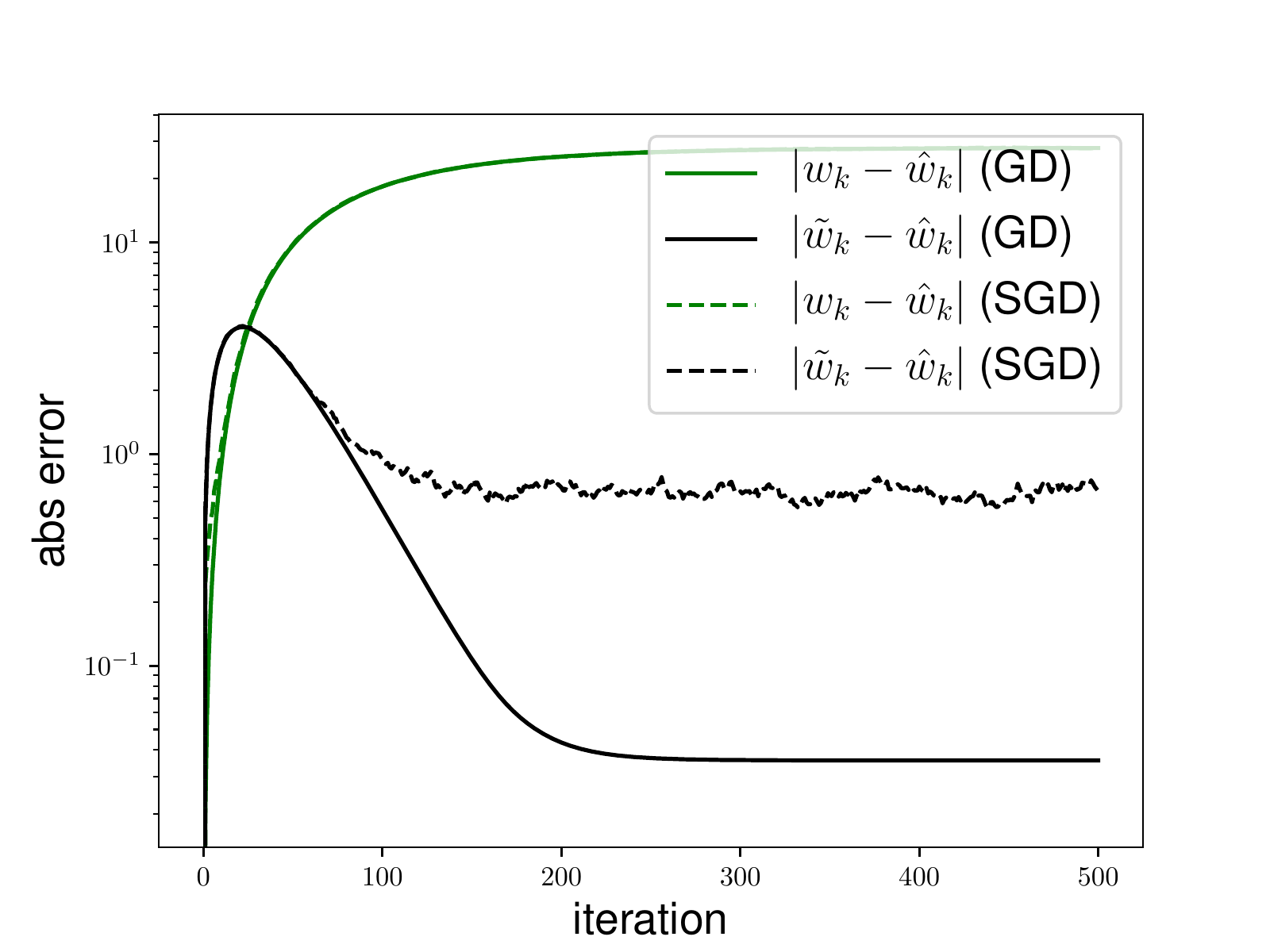} & 
\includegraphics[width=0.3\linewidth]{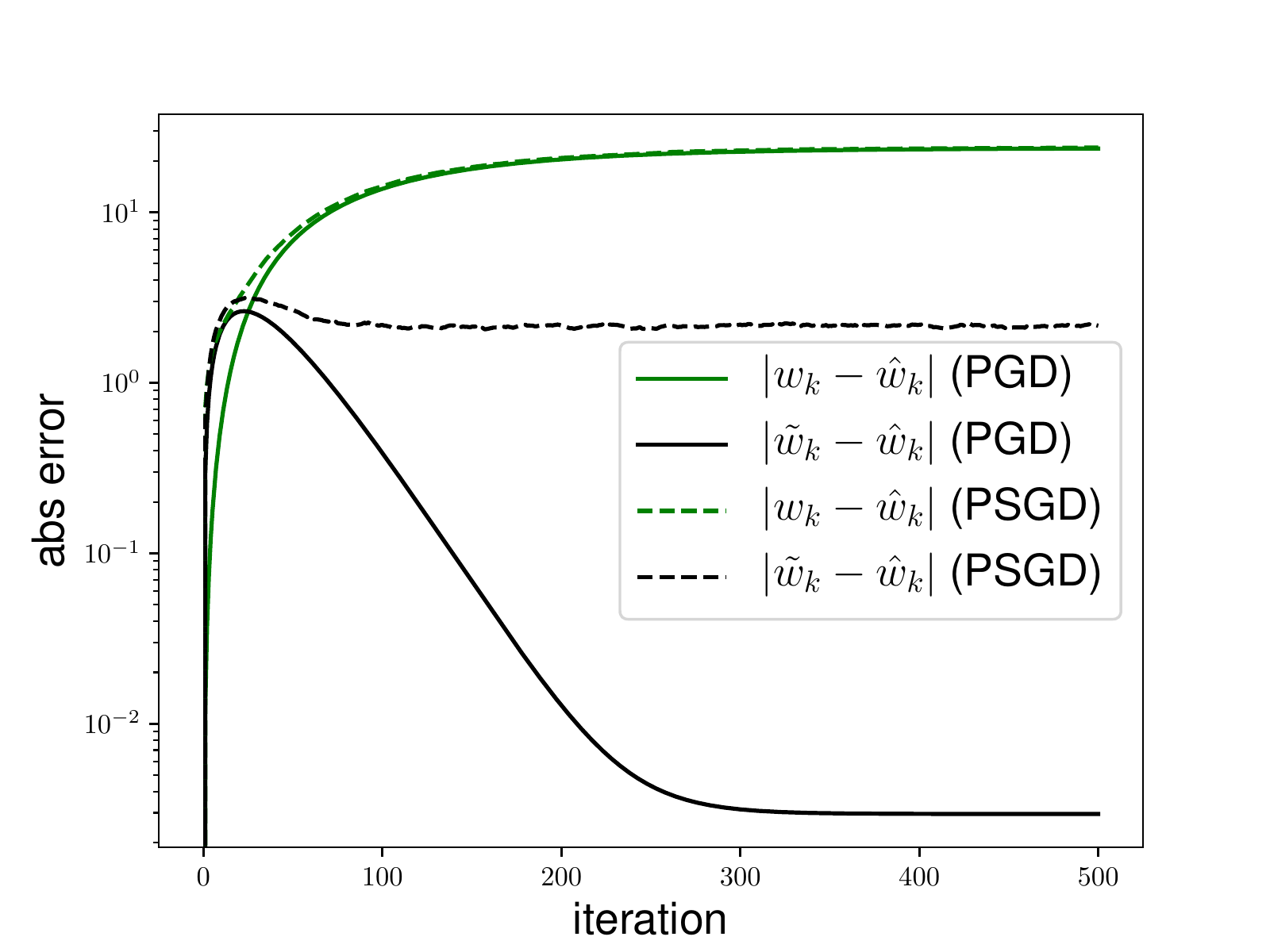} &
\includegraphics[width=0.3\linewidth]{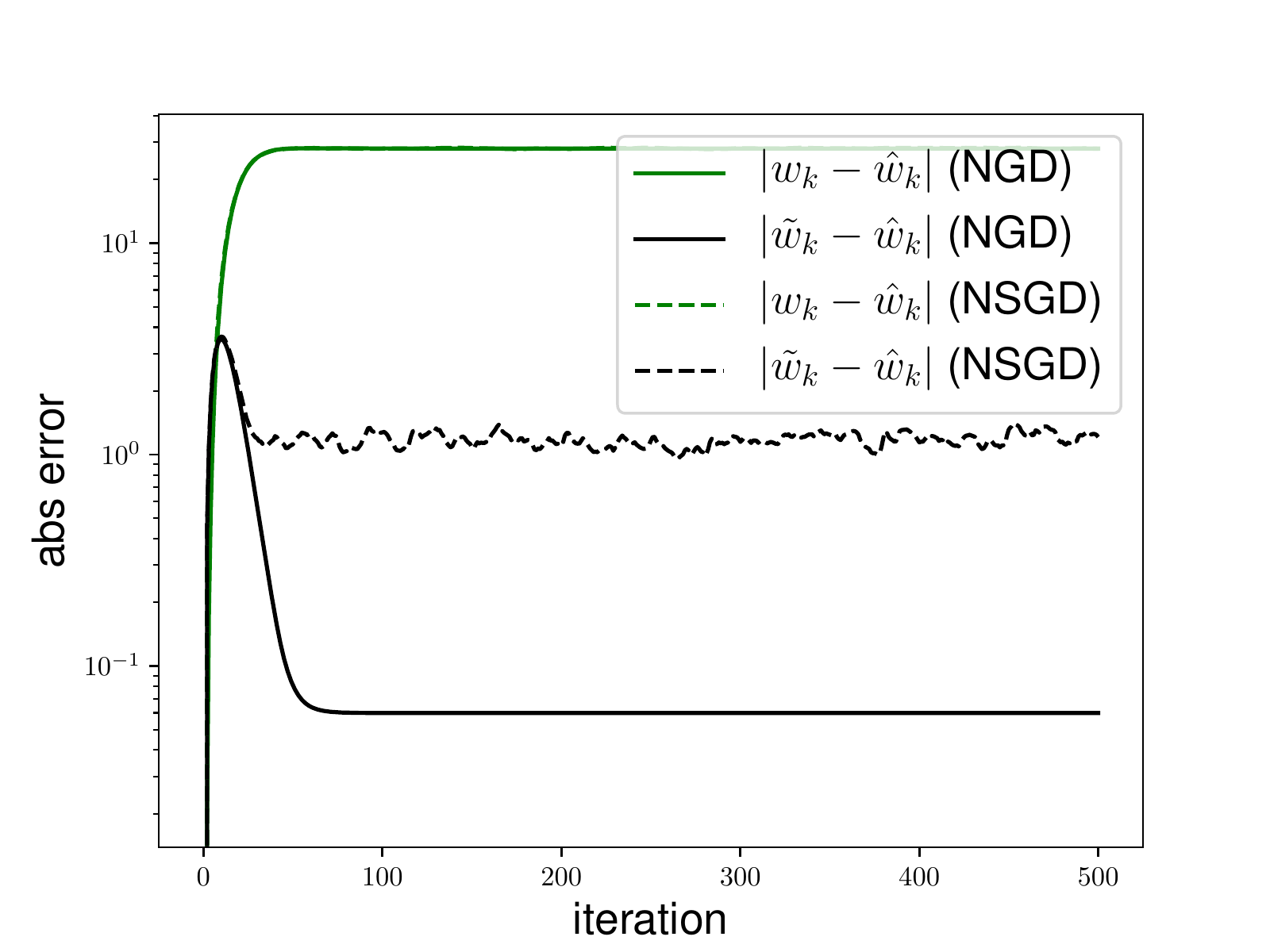} \\
(a) GD and SGD & (b) PGD and PSGD & (c) NGD and NSGD
\end{tabular}
\caption{\small
Logistic regression on MNIST dataset.
X-axis: iteration; y-axis: the absolute approximation error in logarithmic scale.
Green lines represent $\norm{{w}_t - \hat{w}_t}_1$ and black lines represent $\norm{\tilde{w}_t - \hat{w}_t}_1$, where $w_t$, $\hat{w}_t$ and $\tilde{w}_t$ are the unregularized path, the regularized path and the iterate averaged path, respectively.
Solid lines and dashed lines are the results obtained by running deterministic and stochastic algorithms respectively.
For both deterministic and stochastic algorithms, we see that the error between $\tilde{w}_t$ and $\hat{w}_t$ has a small upper bound.
Moreover, the error bounds for the paths generated by stochastic algorithms are relatively bigger.
}
\label{fig:mnist-logistic-regression}
\end{figure*}

\paragraph{Logistic regression}
Next we set the loss function $L(w)$ to be the logistic regression objective with a small $\ell_2$-regularizer, which is then strongly convex and smooth, as required by Theorem~\ref{thm:sgd-general}.
We firstly generate the unregularized paths and perform iterate averaging over them.
Next, since it is impossible to visualize a high dimensional cubic with vertices decided by Theorem~\ref{thm:sgd-general},
instead we set $\lambda={1}/{\gamma}- {1}/{\eta}$, and add an extra regularization term with this particular hyperparameter to obtain the regularized paths.
Lastly we measure the errors between the averaged solutions and the regularized solutions to verify the effect of iterate averaging applied on strongly convex and smooth loss functions.
The learning rates are set to be constant.
The approximation errors are plotted in Figure~\ref{fig:mnist-logistic-regression}.

In Figure~\ref{fig:mnist-logistic-regression} (a),
the solid black line measures the error between the averaged GD path and the regularized GD path,
and indeed the error is bounded and small as predicted by Theorem~\ref{thm:sgd-general};
the dashed black line is the result obtained by running SGD, which suggests that the approximation error, though increases a little due to randomness, is still small.

For completeness, we also test P(S)GD and N(S)GD with results shown in Figure~\ref{fig:mnist-logistic-regression} (b) and (c).
For P(S)GD, we use the Hessian in linear regression experiments as the preconditioning matrix (since the Hessian of the logistic loss varies during training).
Figure~\ref{fig:mnist-logistic-regression} (b) and (c) show that the averaged solutions approximately achieve the generalized/vanilla $\ell_2$-regularization effects respectively. 
And for the stochastic optimization paths, the approximation errors between the averaged paths and the regularized paths increase by a small amount due to the randomness of the algorithms.

\subsection{Application in deep neural networks}\label{sec:exp-dnn}

Lastly, we study the benefits of using iterate averaging in modern deep neural networks.

We train VGG-16~\cite{simonyan2014very} and ResNet-18~\cite{He_2016} on CIFAR-10 and CIFAR-100 datasets, with standard tricks including batch normalization, data augmentation, learning rate decay and weight decay. 
All experiments are repeated three times to obtain the mean and deviation.
The running times are measured by performing the experiments using a single GPU K80.
The models are trained for $300$ epochs using SGD.
We perform ``epoch averaging'' using the $240$ checkpoints saved from the $61$st to the $300$th epoch.
The first $60$ epochs are skipped since the models in the early phase are extremely unstable.
After averaging the parameters, we apply a trick proposed by~\citet{izmailov2018averaging} to handle the batch normalization statistics which are not trained by SGD.
Specifically, we make a forward pass on the training data to compute the activation statistics for the batch normalization layers.
For the choice of averaging scheme, we  test standard geometric distribution with success probability $p\in \{0.9999, 0.999, 0.99, 0.9\}$.

\begin{table}
\centering
\caption{\small CIFAR-10 and CIFAR-100 experiments}
\label{tab:cifar}
\begin{tabular}{c|c|c|c}
\hline
Dataset & \multicolumn{2}{c|}{CIFAR-10} & CIFAR-100\\
\hline
Model & {\small VGG-16} & {\small ResNet-18} & {\small ResNet-18}\\
\hhline{====}
\makecell{\small Accuracy after \\training ($\%$)} & \makecell{$92.54$\\$\pm 0.22$} & \makecell{$94.54$\\$\pm 0.04$} & \makecell{$75.62$\\$\pm 0.16$} \\
\hline
\makecell{\small Accuracy after \\averaging ($\%$)} & \makecell{$\bm{93.18}$\\$\pm 0.06$} & \makecell{$\bm{94.72}$\\$\pm 0.04$} & \makecell{$\bm{76.24}$\\$\pm 0.05$}\\
\hhline{====}
{\small Time of training} & $\sim 4.5$h & $\sim 8.3$h & $\sim 8.3$h \\
\hline
{\small Time of averaging\footnotemark} & $\sim 47$s & $\sim 56$s & $\sim 58$s \\
\hline
\end{tabular}
\end{table}
\footnotetext{The {time of averaging} contains the {time of IO} and {fixing BN}, which takes the major overhead. For example, in CIFAR-10 and VGG-16 experiments, IO takes $\sim 22$s, fixing BN takes $\sim 18$s, while performing averaging and evaluation take merely $\sim 7$s.}

The results are shown in Table~\ref{tab:cifar}.
We see that
(i) averaging the SGD path does improve performance since it introduce an implicit regularization by our understanding;
(ii) obtaining such regularization by iterate averaging is computationally cheap. 
It only takes \emph{a few seconds} to test a hyperparameter of the averaging scheme.
In contrast, several hours are required to test a hyperparameter for traditional explicit regularization since it requires re-training the model.
Finally, we emphasize that the space cost of our method is also affordable.
In fact, in our experiments, we perform epoch-wise averaging instead of iterate-wise averaging, 
thus we only need to store  a few hundreds of the checkpoints.

\section{Discussion}

\paragraph{$\ell_1$-regularization}
Notice that all our results obtain $\ell_2$-type regularization effects.
A natural follow-up question would be whether or not there is an averaging scheme that acts as an $\ell_1$-regularizer.
However, we here provide some evidence that this question is relatively hard.
As illustrated in Figure~\ref{fig:l1-regupath}, even for simple quadratic loss, the $\ell_1$-regularized solutions could lie outside of the convex hull of a SGD path.
Therefore, any averaging scheme with positive weights fails to obtain such $\ell_1$-regularized solutions.

\begin{figure}
\centering
\includegraphics[width=0.85\linewidth]{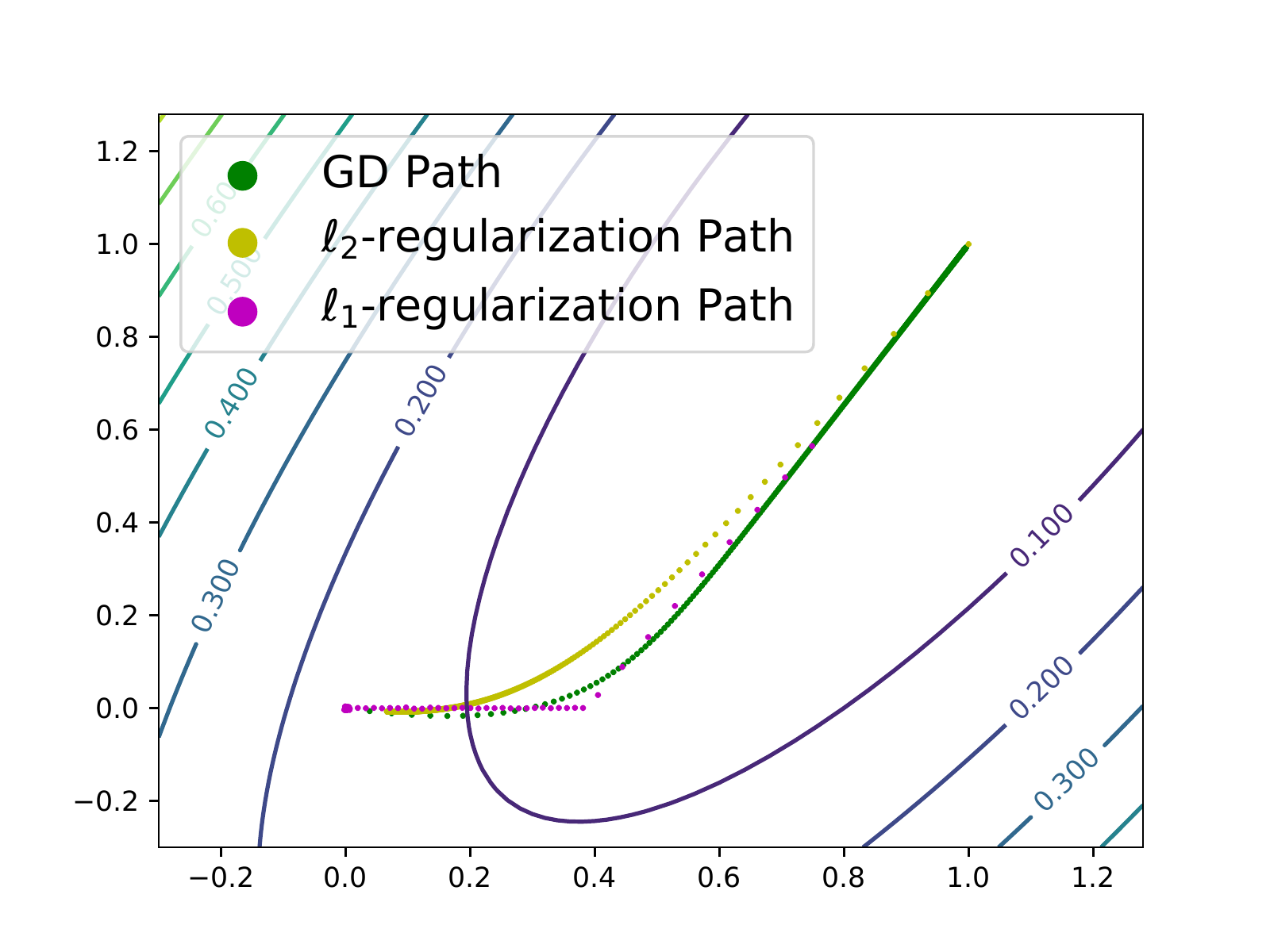}
\caption{\small 
A 2-D demonstration of the $\ell_1$-regularization path.
{Green dots}: the vanilla GD path $w_t$;
{yellow dots}: the $\ell_2$-regularization path $\hat{w}_{{\lambda}, \ell_2}$;
{purple dots}: the $\ell_1$-regularization path $\hat{w}_{{\lambda}, \ell_1}$.
There exist $\ell_1$-regularized solutions outside of the convex hull of the GD path, while all of the $\ell_2$-regularized solutions are inside of that.
}
\label{fig:l1-regupath}
\end{figure}

\paragraph{Infinite width neural network}
Recent works suggest that a sufficient wide neural network trained by SGD behaves like a quadratic model, i.e., the neural tangent kernel (NTK)~\cite{jacot2018neural,arora2019exact,cao2019generalization}.
Nonetheless, the NTK approximation fails when there is an explicit $\ell_2$-regularizer~\cite{wei2019regularization}.
Since our results hold for kernel ridge regression, we conjecture that iterate averaging could be a potential approach to achieve $\ell_2$-regularization for the NTK regime.
We leave further investigation of this issue in future works.

\section{Conclusions}
In this work, we establish averaging schemes for various optimization methods and objective functions to obtain adjustable $\ell_2$-type regularization effects, i.e., SGD with preconditioning and adaptive learning rate schedules, Nesterov's accelerated stochastic gradient descent, and  strongly convex and smooth objective functions.
Particularly, we resolve an open question in~\cite{neu2018iterate}.
The method of achieving regularization by iterate averaging requires little computation.
It is further shown experimentally that iterate averaging even benefits practical deep learning models.
Our theoretical and empirical results demonstrate the potential of adopting iterate averaging to obtain adjustable regularization for free in a much broader class of optimization methods and objective functions.

\section*{Acknowledgement}
This research is supported in part by NSF CAREER grant 1652257, ONR Award
N00014-18-1-2364 and the Lifelong Learning Machines program from DARPA/MTO.

\bibliography{references}

\begin{thebibliography}{43}
\providecommand{\natexlab}[1]{#1}
\providecommand{\url}[1]{\texttt{#1}}
\expandafter\ifx\csname urlstyle\endcsname\relax
  \providecommand{\doi}[1]{doi: #1}\else
  \providecommand{\doi}{doi: \begingroup \urlstyle{rm}\Url}\fi

\bibitem[Arora et~al.(2019)Arora, Du, Hu, Li, Salakhutdinov, and
  Wang]{arora2019exact}
Arora, S., Du, S.~S., Hu, W., Li, Z., Salakhutdinov, R., and Wang, R.
\newblock On exact computation with an infinitely wide neural net.
\newblock \emph{arXiv preprint arXiv:1904.11955}, 2019.

\bibitem[Bach \& Moulines(2013)Bach and Moulines]{bach2013non}
Bach, F. and Moulines, E.
\newblock Non-strongly-convex smooth stochastic approximation with convergence
  rate o (1/n).
\newblock In \emph{Advances in neural information processing systems}, pp.\
  773--781, 2013.

\bibitem[Beck \& Teboulle(2009)Beck and Teboulle]{beck2009fast}
Beck, A. and Teboulle, M.
\newblock A fast iterative shrinkage-thresholding algorithm for linear inverse
  problems.
\newblock \emph{SIAM journal on imaging sciences}, 2\penalty0 (1):\penalty0
  183--202, 2009.

\bibitem[Bottou \& Bousquet(2008)Bottou and Bousquet]{bottou2008tradeoffs}
Bottou, L. and Bousquet, O.
\newblock The tradeoffs of large scale learning.
\newblock In \emph{Advances in neural information processing systems}, pp.\
  161--168, 2008.

\bibitem[Cai et~al.(2018)Cai, Li, and Shen]{cai2018quantitative}
Cai, Y., Li, Q., and Shen, Z.
\newblock A quantitative analysis of the effect of batch normalization on
  gradient descent.
\newblock \emph{arXiv preprint arXiv:1810.00122}, 2018.

\bibitem[Cao \& Gu(2019)Cao and Gu]{cao2019generalization}
Cao, Y. and Gu, Q.
\newblock Generalization bounds of stochastic gradient descent for wide and
  deep neural networks.
\newblock \emph{arXiv preprint arXiv:1905.13210}, 2019.

\bibitem[Clark(1987)]{clark1987short}
Clark, D.~S.
\newblock Short proof of a discrete gronwall inequality.
\newblock \emph{Discrete applied mathematics}, 16\penalty0 (3):\penalty0
  279--281, 1987.

\bibitem[Dennis~Jr \& Schnabel(1996)Dennis~Jr and
  Schnabel]{dennis1996numerical}
Dennis~Jr, J.~E. and Schnabel, R.~B.
\newblock \emph{Numerical methods for unconstrained optimization and nonlinear
  equations}, volume~16.
\newblock Siam, 1996.

\bibitem[Devlin et~al.(2018)Devlin, Chang, Lee, and Toutanova]{devlin2018bert}
Devlin, J., Chang, M.-W., Lee, K., and Toutanova, K.
\newblock Bert: Pre-training of deep bidirectional transformers for language
  understanding.
\newblock \emph{arXiv preprint arXiv:1810.04805}, 2018.

\bibitem[Grandvalet \& Bengio(2005)Grandvalet and Bengio]{grandvalet2005semi}
Grandvalet, Y. and Bengio, Y.
\newblock Semi-supervised learning by entropy minimization.
\newblock In \emph{Advances in neural information processing systems}, pp.\
  529--536, 2005.

\bibitem[Granziol et~al.(2020)Granziol, Wan, and Roberts]{granziol2020iterate}
Granziol, D., Wan, X., and Roberts, S.
\newblock Iterate averaging helps: An alternative perspective in deep learning.
\newblock \emph{arXiv preprint arXiv:2003.01247}, 2020.

\bibitem[Gunasekar et~al.(2018)Gunasekar, Lee, Soudry, and
  Srebro]{gunasekar2018characterizing}
Gunasekar, S., Lee, J., Soudry, D., and Srebro, N.
\newblock Characterizing implicit bias in terms of optimization geometry.
\newblock \emph{arXiv preprint arXiv:1802.08246}, 2018.

\bibitem[He et~al.(2015)He, Zhang, Ren, and Sun]{he2015delving}
He, K., Zhang, X., Ren, S., and Sun, J.
\newblock Delving deep into rectifiers: Surpassing human-level performance on
  imagenet classification.
\newblock In \emph{Proceedings of the IEEE international conference on computer
  vision}, pp.\  1026--1034, 2015.

\bibitem[He et~al.(2016)He, Zhang, Ren, and Sun]{He_2016}
He, K., Zhang, X., Ren, S., and Sun, J.
\newblock Deep residual learning for image recognition.
\newblock \emph{2016 IEEE Conference on Computer Vision and Pattern Recognition
  (CVPR)}, Jun 2016.
\newblock \doi{10.1109/cvpr.2016.90}.
\newblock URL \url{http://dx.doi.org/10.1109/CVPR.2016.90}.

\bibitem[Hu et~al.(2017{\natexlab{a}})Hu, Li, Li, and Liu]{hu2017diffusion}
Hu, W., Li, C.~J., Li, L., and Liu, J.-G.
\newblock On the diffusion approximation of nonconvex stochastic gradient
  descent.
\newblock \emph{arXiv preprint arXiv:1705.07562}, 2017{\natexlab{a}}.

\bibitem[Hu et~al.(2017{\natexlab{b}})Hu, Li, and Su]{hu2017global}
Hu, W., Li, C.~J., and Su, W.
\newblock On the global convergence of a randomly perturbed dissipative
  nonlinear oscillator.
\newblock \emph{arXiv preprint arXiv:1712.05733}, 2017{\natexlab{b}}.

\bibitem[Hu et~al.(2020)Hu, Xiao, and Pennington]{hu2020provable}
Hu, W., Xiao, L., and Pennington, J.
\newblock Provable benefit of orthogonal initialization in optimizing deep
  linear networks.
\newblock \emph{arXiv preprint arXiv:2001.05992}, 2020.

\bibitem[Ioffe \& Szegedy(2015)Ioffe and Szegedy]{ioffe2015batch}
Ioffe, S. and Szegedy, C.
\newblock Batch normalization: Accelerating deep network training by reducing
  internal covariate shift.
\newblock \emph{arXiv preprint arXiv:1502.03167}, 2015.

\bibitem[Izmailov et~al.(2018)Izmailov, Podoprikhin, Garipov, Vetrov, and
  Wilson]{izmailov2018averaging}
Izmailov, P., Podoprikhin, D., Garipov, T., Vetrov, D., and Wilson, A.~G.
\newblock Averaging weights leads to wider optima and better generalization.
\newblock \emph{arXiv preprint arXiv:1803.05407}, 2018.

\bibitem[Jacot et~al.(2018)Jacot, Gabriel, and Hongler]{jacot2018neural}
Jacot, A., Gabriel, F., and Hongler, C.
\newblock Neural tangent kernel: Convergence and generalization in neural
  networks.
\newblock In \emph{Advances in neural information processing systems}, pp.\
  8571--8580, 2018.

\bibitem[Jain et~al.(2018)Jain, Kakade, Kidambi, Netrapalli, and
  Sidford]{jain2018parallelizing}
Jain, P., Kakade, S., Kidambi, R., Netrapalli, P., and Sidford, A.
\newblock Parallelizing stochastic gradient descent for least squares
  regression: mini-batching, averaging, and model misspecification.
\newblock \emph{Journal of Machine Learning Research}, 18, 2018.

\bibitem[Krogh \& Hertz(1992)Krogh and Hertz]{krogh1992simple}
Krogh, A. and Hertz, J.~A.
\newblock A simple weight decay can improve generalization.
\newblock In \emph{Advances in neural information processing systems}, pp.\
  950--957, 1992.

\bibitem[Lakshminarayanan \& Szepesvari(2018)Lakshminarayanan and
  Szepesvari]{lakshminarayanan2018linear}
Lakshminarayanan, C. and Szepesvari, C.
\newblock Linear stochastic approximation: How far does constant step-size and
  iterate averaging go?
\newblock In \emph{International Conference on Artificial Intelligence and
  Statistics}, pp.\  1347--1355, 2018.

\bibitem[Li et~al.(2017)Li, Tai, et~al.]{li2017stochastic}
Li, Q., Tai, C., et~al.
\newblock Stochastic modified equations and adaptive stochastic gradient
  algorithms.
\newblock In \emph{Proceedings of the 34th International Conference on Machine
  Learning-Volume 70}, pp.\  2101--2110. JMLR. org, 2017.

\bibitem[Martens(2014)]{martens2014new}
Martens, J.
\newblock New insights and perspectives on the natural gradient method.
\newblock \emph{arXiv preprint arXiv:1412.1193}, 2014.

\bibitem[Mohri et~al.(2018)Mohri, Rostamizadeh, and
  Talwalkar]{mohri2018foundations}
Mohri, M., Rostamizadeh, A., and Talwalkar, A.
\newblock \emph{Foundations of machine learning}.
\newblock MIT press, 2018.

\bibitem[Nesterov(1983)]{nesterov1983method}
Nesterov, Y.~E.
\newblock A method for solving the convex programming problem with convergence
  rate o (1/k\^{} 2).
\newblock In \emph{Dokl. akad. nauk Sssr}, volume 269, pp.\  543--547, 1983.

\bibitem[Neu \& Rosasco(2018)Neu and Rosasco]{neu2018iterate}
Neu, G. and Rosasco, L.
\newblock Iterate averaging as regularization for stochastic gradient descent.
\newblock \emph{arXiv preprint arXiv:1802.08009}, 2018.

\bibitem[Shi et~al.(2019)Shi, Du, Su, and Jordan]{shi2019acceleration}
Shi, B., Du, S.~S., Su, W., and Jordan, M.~I.
\newblock Acceleration via symplectic discretization of high-resolution
  differential equations.
\newblock In \emph{Advances in Neural Information Processing Systems}, pp.\
  5745--5753, 2019.

\bibitem[Silver et~al.(2017)Silver, Schrittwieser, Simonyan, Antonoglou, Huang,
  Guez, Hubert, Baker, Lai, Bolton, et~al.]{silver2017mastering}
Silver, D., Schrittwieser, J., Simonyan, K., Antonoglou, I., Huang, A., Guez,
  A., Hubert, T., Baker, L., Lai, M., Bolton, A., et~al.
\newblock Mastering the game of go without human knowledge.
\newblock \emph{Nature}, 550\penalty0 (7676):\penalty0 354, 2017.

\bibitem[Simonyan \& Zisserman(2014)Simonyan and Zisserman]{simonyan2014very}
Simonyan, K. and Zisserman, A.
\newblock Very deep convolutional networks for large-scale image recognition.
\newblock \emph{arXiv preprint arXiv:1409.1556}, 2014.

\bibitem[Soudry et~al.(2018)Soudry, Hoffer, Nacson, Gunasekar, and
  Srebro]{soudry2018implicit}
Soudry, D., Hoffer, E., Nacson, M.~S., Gunasekar, S., and Srebro, N.
\newblock The implicit bias of gradient descent on separable data.
\newblock \emph{The Journal of Machine Learning Research}, 19\penalty0
  (1):\penalty0 2822--2878, 2018.

\bibitem[Su et~al.(2014)Su, Boyd, and Candes]{su2014differential}
Su, W., Boyd, S., and Candes, E.
\newblock A differential equation for modeling nesterov’s accelerated
  gradient method: Theory and insights.
\newblock In \emph{Advances in Neural Information Processing Systems}, pp.\
  2510--2518, 2014.

\bibitem[Suggala et~al.(2018)Suggala, Prasad, and
  Ravikumar]{suggala2018connecting}
Suggala, A., Prasad, A., and Ravikumar, P.~K.
\newblock Connecting optimization and regularization paths.
\newblock In \emph{Advances in Neural Information Processing Systems}, pp.\
  10608--10619, 2018.

\bibitem[Tibshirani(1996)]{tibshirani1996regression}
Tibshirani, R.
\newblock Regression shrinkage and selection via the lasso.
\newblock \emph{Journal of the Royal Statistical Society: Series B
  (Methodological)}, 58\penalty0 (1):\penalty0 267--288, 1996.

\bibitem[Tikhonov \& Arsenin(1977)Tikhonov and Arsenin]{tikhonov1977solutions}
Tikhonov, A.~N. and Arsenin, V.~Y.
\newblock \emph{Solutions of ill-posed problems}.
\newblock V. H. Winston \& Sons, Washington, D.C.: John Wiley \& Sons, New
  York, 1977.
\newblock Translated from the Russian, Preface by translation editor Fritz
  John, Scripta Series in Mathematics.

\bibitem[Wei et~al.(2019)Wei, Lee, Liu, and Ma]{wei2019regularization}
Wei, C., Lee, J.~D., Liu, Q., and Ma, T.
\newblock Regularization matters: Generalization and optimization of neural
  nets vs their induced kernel.
\newblock In \emph{Advances in Neural Information Processing Systems}, pp.\
  9709--9721, 2019.

\bibitem[Wilson et~al.(2017)Wilson, Roelofs, Stern, Srebro, and
  Recht]{wilson2017marginal}
Wilson, A.~C., Roelofs, R., Stern, M., Srebro, N., and Recht, B.
\newblock The marginal value of adaptive gradient methods in machine learning.
\newblock In \emph{Advances in Neural Information Processing Systems}, pp.\
  4148--4158, 2017.

\bibitem[Yang et~al.(2018)Yang, Arora, Zhao, et~al.]{yang2018physical}
Yang, L., Arora, R., Zhao, T., et~al.
\newblock The physical systems behind optimization algorithms.
\newblock In \emph{Advances in Neural Information Processing Systems}, pp.\
  4372--4381, 2018.

\bibitem[Zhang et~al.(2016)Zhang, Bengio, Hardt, Recht, and
  Vinyals]{zhang2016understanding}
Zhang, C., Bengio, S., Hardt, M., Recht, B., and Vinyals, O.
\newblock Understanding deep learning requires rethinking generalization.
\newblock \emph{arXiv preprint arXiv:1611.03530}, 2016.

\bibitem[Zhang et~al.(2019)Zhang, Lucas, Ba, and Hinton]{zhang2019lookahead}
Zhang, M., Lucas, J., Ba, J., and Hinton, G.~E.
\newblock Lookahead optimizer: k steps forward, 1 step back.
\newblock In \emph{Advances in Neural Information Processing Systems}, pp.\
  9593--9604, 2019.

\bibitem[Zhou(2018)]{zhou2018fenchel}
Zhou, X.
\newblock On the fenchel duality between strong convexity and lipschitz
  continuous gradient.
\newblock \emph{arXiv preprint arXiv:1803.06573}, 2018.

\bibitem[Zhu et~al.(2018)Zhu, Wu, Yu, Wu, and Ma]{zhu2018anisotropic}
Zhu, Z., Wu, J., Yu, B., Wu, L., and Ma, J.
\newblock The anisotropic noise in stochastic gradient descent: Its behavior of
  escaping from minima and regularization effects.
\newblock \emph{arXiv preprint arXiv:1803.00195}, 2018.

\end{thebibliography}
\bibliographystyle{icml2020}

\onecolumn
\appendix

\section{Continuous analysis}\label{sec:continuous-analysis}
To motivate our proofs for the theorems in main text, let us first elaborate the continuous cases.
Then we will extend our analysis to the discrete circumstances.
One can safely skip this part and go directly to Section~\ref{sec:discrete-analysis} for the missing proofs in main text, which is self-consistent.

\paragraph{Continuous optimization paths}
To ease notations and preliminaries, in this part we only discuss gradient descent (GD) and Nesterov's accelerated gradient descent (NGD), and their strong continuous approximation via ordinary differential equations (ODEs).
For SGD and NSGD, existing works show that there are weak continuous approximation by stochastic differential equations (SDEs)~\cite{hu2017diffusion,hu2017global,li2017stochastic}.
Our analysis can be extended to SDEs, but we believe it serves better to motivate our discrete proofs by focusing on ODEs.

We consider loss $L(w)$ and $\ell_2$-regularizer $R(w)=\half \norm{w}_2^2$.
Let the learning rate $\eta\to 0$, the path of $L(w)$ optimized by GD converges to the following ODE~\cite{yang2018physical}
\begin{equation*}
    \dif w_t = -\grad L(w_t)\dif t.
\end{equation*}
Similarly the continuous GD optimization path of regularized loss admits
\begin{equation*}
    \dif \hat{w}_t = -\bracket{\grad L(\hat{w}_t)+\lambda \hat{w}_t}\dif t.
\end{equation*}

As for NGD, \citet{su2014differential,yang2018physical} show if the loss is $\alpha$-strongly convex, then the NGD optimization path converges to
\begin{equation*}
    w''_t+2\sqrt{\alpha}w'_t+L'(w_t) = 0.
\end{equation*}
Since $\hat{L}(\hat{w}) = L(\hat{w})+\frac{\lambda}{2}\norm{\hat{w}}_2^2$ is $(\alpha+\lambda)$-strongly convex, the NGD path of the regularized loss satisfies
\begin{equation*}
    \hat{w}''_t+2\sqrt{\alpha+\lambda}\hat{w}'_t+L'(\hat{w}_t)+\lambda\hat{w}_t = 0.
\end{equation*}

\paragraph{Continuous weighting scheme}
We define the continuous weighting scheme as
\begin{equation*}
    p_t\ge 0,\quad t\ge 0, \quad P_t = \int_0^t p(s)\dif s,\quad \lim_{t\to \infty}P_t = 1.
\end{equation*}

\begin{lem}\label{thm:continuous-sufficient-condition}
Given two continuous dynamic $x_t,\ \hat{x}_t,\ t\ge 0$. Let $\tilde{x}_t = P_t^\inv \int_0^t p_s x_s\dif s$. Suppose $x_0=\hat{x}_0=0$. If the continuous weighting scheme $P_t$ satisfies
\begin{equation*}
    \dif \hat{x}_t = (1-P_t)\dif x_t, \quad t\ge 0,
\end{equation*}
then we have
\begin{equation*}
    P_t (x_t-\tilde{x}_t) = x_t - \hat{x}_t,\quad t\ge 0,
\end{equation*}
and
\begin{equation*}
    \hat{x}_t - \tilde{x}_t = (1-P_t)(x_t - \tilde{x}_t), \quad t\ge 0.
\end{equation*}
\end{lem}
\begin{proof}
By definition we have for $t\ge 0$,
\begin{equation*}
\begin{aligned}
    &\tilde{x}_t = P_t^\inv \int_0^t p_s x_s\dif s = P_t^\inv \bracket{ \eval{x_s P_s}_0^t -\int_0^t P_s\dif x_s } = x_t - P_t^\inv \int_0^t P_s \dif x_s\\
    =& x_t - P_t^\inv \bracket{ x_t - \int_0^t (1-P_s)\dif x_s } = x_t - P_t^\inv \bracket{ x_t - \int_0^t \dif \hat{x}_s } \\
    =& x_t - P_t^\inv \bracket{x_t - \hat{x}_t}.
\end{aligned}
\end{equation*}
Thus
\begin{equation*}
    P_t (x_t-\tilde{x}_t) = x_t - \hat{x}_t,
\end{equation*}
and
\begin{equation*}
    \hat{x}_t - \tilde{x}_t = x_t - P_t (x_t-\tilde{x}_t) - \tilde{x}_t  = (1-P_t)(x_t - \tilde{x}_t).
\end{equation*}
\end{proof}

\subsection{Continuous Theorem~\ref{thm:sgd-linear}}
Consider linear regression problem $L(w)=\frac{1}{2n}\sum_{i=1}^n\norm{w^\top x_i-y_i}_2^2 = \frac{1}{2}w^\top \Sigma w-w^\top a + \text{const}$, and $\ell_2$-regularizer $R(w)=\half\norm{w}_2^2$.
Assume the initial condition $w_0 = \hat{w}_0 = 0$, then the GD dynamics for the unregularized and regularized losses are
\begin{align*}
    & \dif w_t = -\bracket{\Sigma w_t - a}\dif t,\quad w_0 = 0, \\
    & \dif \hat{w}_t = -\bracket{\Sigma \hat{w}_t - a + \lambda \hat{w}_t}\dif t,\quad \hat{w}_0 = 0.
\end{align*}
The ODEs are solved by
\begin{equation*}
    w_t = \bracket{I - e^{-\Sigma t}} \Sigma^{-1} a, \quad
    \hat{w}_t = \bracket{I-e^{-(\Sigma+\lambda I )t}}{\bracket{\Sigma+\lambda I}^\inv} a.
\end{equation*}

Now let the continuous weighting scheme be
\begin{equation*}
    P_t = 1 - e^{\lambda t},
\end{equation*}
then we have 
\begin{equation*}
    \dif \hat{w}_t = (1-P_t) \dif w_t,
\end{equation*}
thus by Lemma~\ref{thm:continuous-sufficient-condition} we obtain
\begin{equation*}
    \hat{w}_t - \tilde{w}_t  = (1-P_t)(w_t - \tilde{w}_t),
\end{equation*}
which proves the continuous version of Theorem~\ref{thm:sgd-linear}.

\subsection{Continuous Theorem~\ref{thm:nsgd-linear}}
Consider linear regression problem $L(w)=\frac{1}{2n}\sum_{i=1}^n\norm{w^\top x_i-y_i}_2^2 = \frac{1}{2}w^\top\Sigma w-w^\top a + \text{const}$, and $\ell_2$-regularizer $R(w)=\half\norm{w}_2^2$.
Assume the initial condition $w_0=w'_0=0$ and $\hat{w}_0=\hat{w}'_0=0$.
Then the unregularized and regularized NGD dynamics are
\begin{align}
    & w''_t + 2\sqrt{\alpha} w'_t + \Sigma w_t - a = 0,\quad w_0 = w'_0 = 0,\label{eq:ngf-linear} \\
    & \hat{w}''_t + 2\sqrt{\alpha+\lambda} \hat{w}'_t + (\Sigma+\lambda) \hat{w}_t - a = 0,\quad \hat{w}_0 = \hat{w}'_0 = 0.\label{eq:ngf-linear-regu}
\end{align}

We first solve the order-2 ODE Eq.~\eqref{eq:ngf-linear} in the canonical way, and then obtain the solution of Eq.~\eqref{eq:ngf-linear-regu} similarly.
To do so, let's firstly ignore the constant term and solve the homogenous ODE of Eq.~\eqref{eq:ngf-linear}, and obtain two general solutions of the homogenous equation as
\begin{equation*}
        w_{t,1} = e^{\sqrt{\alpha}t}\cos{\sqrt{\Sigma-\alpha}t},\quad 
        w_{t,2} = e^{\sqrt{\alpha}t}\sin{\sqrt{\Sigma-\alpha}t}.
\end{equation*}
Then we guess a particular solution of Eq.~\eqref{eq:ngf-linear} as $w_{t,0} = \Sigma^\inv a$. Thus the general solution of ODE~\eqref{eq:ngf-linear} can be decomposed as $w_t = \lambda_1 w_{t,1} + \lambda_2 w_{t,2} + w_{t,0}$. 
Consider the initial conditions $w_0 = w'_0 = 0$, we obtain $\lambda_1 = -{\Sigma^\inv}a,\ \lambda_2 = -{\Sigma^\inv}a\sqrt{{(\Sigma-\alpha)^\inv}\alpha}$.
Thus the solution of Eq.~\eqref{eq:ngf-linear} is
\begin{equation}\label{eq:ngf-linear-solution}
    \begin{aligned}
    w_t &= {\Sigma^\inv}a \bracket{ 1 - e^{-\sqrt{\alpha}t}\cos{\sqrt{\Sigma-\alpha}t} - \sqrt{\alpha{(\Sigma-\alpha)^\inv}}e^{-\sqrt{\alpha}t}\sin{\sqrt{\Sigma-\alpha}t} },\\
    w'_t &= a\sqrt{{(\Sigma-\alpha)^\inv}}e^{-\sqrt{\alpha}t}\sin{\sqrt{\Sigma-\alpha}t}.
    \end{aligned}
\end{equation}
Repeat these procedures, Eq.~\eqref{eq:ngf-linear-solution} is solved by
\begin{equation}\label{eq:ngf-linear-regu-solution}
    \begin{aligned}
    \hat{w}_t &= {(\Sigma+\lambda)^\inv}a \bracket{ 1 - e^{-\sqrt{\alpha+\lambda}t}\cos{\sqrt{\Sigma-\alpha}t} - \sqrt{\bracket{\alpha+\lambda}{(\Sigma-\alpha)^\inv}}e^{-\sqrt{\alpha+\lambda}t}\sin{\sqrt{\Sigma-\alpha}t} },\\
    \hat{w}'_t &= a{\sqrt{{(\Sigma-\alpha)^\inv}}}e^{-\sqrt{\alpha+\lambda}t}\sin{\sqrt{\Sigma-\alpha}t}.
    \end{aligned}
\end{equation}

Now let the continuous weighting scheme be
\begin{equation*}
    P_t = 1 - e^{-(\sqrt{\alpha+\lambda}-\sqrt{\lambda})t},
\end{equation*}
then we have
\begin{equation*}
    \dif \hat{w}_t = (1-P_t)\dif w_t,
\end{equation*}
thus by Lemma~\ref{thm:continuous-sufficient-condition} we obtain
\begin{equation*}
    \hat{w}_t - \tilde{w}_t  = (1-P_t)(w_t - \tilde{w}_t),
\end{equation*}
which proves the continuous version of Theorem~\ref{thm:nsgd-linear}.

\subsection{Continuous Theorem~\ref{thm:sgd-general}}
Consider an $\alpha$-strongly convex and $\beta$-smooth loss function $L(w)$, and $\ell_2$-regularizer.
Without loss of generality assume the minimum of $L(w)$ satisfies $w_* > w_0 = 0$.
Then by Lemma~\ref{thm:1dim-alpha-beta} we have
\begin{equation*}
    \alpha w - b \le \grad L(w) \le \beta w - b,\quad 
    \forall w \in (0, w_*),
\end{equation*}
where $b = -\grad L(0)$, and ``$\le$'' is defined entry-wisely. 
We study the continuous optimization paths caused by GD.

Consider the following three dynamics:
\begin{equation*}
    \dif w_t = -\grad L(w_t) \dif t, \quad
    \dif u_t = - (\alpha u_t - b) \dif t,\quad
    \dif v_t = - (\beta v_t - b) \dif t,\quad
    w_0 = u_0 = v_0 = 0.
\end{equation*}
By the comparison theorem of ODEs (Gronwall's inequality), and solution of linear ODEs, we claim that for all $t>0$, 
\begin{equation}\label{eq:control-gd-path}
    v_t\le w_t\le u_t, \quad u_t = \frac{b}{\alpha}(1-e^{-\alpha t}), \quad
    v_t = \frac{b}{\beta}(1-e^{-\beta t}).
\end{equation}
In a similar manner, for the following three dynamics of regularized loss:
\begin{equation*}
\begin{aligned}
    &\dif \hat{w}_{t,\lambda} = -(\grad L(\hat{w}_{t,\lambda})+\lambda \hat{w}_{t,\lambda}) \dif t,\quad
    &\dif \hat{u}_{t,\lambda} = - ((\lambda+\alpha) \hat{u}_{t,\lambda} - b) \dif t,\quad 
    \dif \hat{v}_{t,\lambda} = - ((\lambda+\beta) \hat{v}_{t,\lambda} - b) \dif t,
\end{aligned}
\end{equation*}
where $\hat{w}_{0,\lambda} = \hat{u}_{0,\lambda} = \hat{v}_{0,\lambda} = 0$.
Similarly we have for all $t>0$,
\begin{equation*}
    \hat{v}_{t,\lambda} \le \hat{w}_{t,\lambda} \le \hat{u}_{t,\lambda},
    \quad \hat{u}_{t,\lambda} = \frac{b}{\lambda+\alpha}(1-e^{-(\lambda+\alpha) t}), 
    \quad \hat{v}_{t,\lambda} = \frac{b}{\lambda+\beta}(1-e^{-(\lambda+\beta) t}).
\end{equation*}
For the continuous weighting scheme
\begin{equation*}
    P_t = 1-e^{-\zeta t},\quad p_t = \zeta e^{-\zeta t},\quad t\ge 0,\quad \zeta >0,
\end{equation*}
the averaged solution is defined as $\tilde{w}_t = {P_t^\inv} \int_0^t p_t w_t\dif t = w_t - {P_t^\inv}\int_0^t P_s\dif w_s$, similar there are $\tilde{u}_t, \tilde{v}_t$.
Thanks to Eq.~\eqref{eq:control-gd-path} and $p_t$ being non-negative, we have
$\tilde{v}_t \le \tilde{w}_t \le \tilde{u}_t$.
Let 
\begin{equation*}
    \lambda_1 = \zeta + \beta -\alpha,\quad \lambda_2 = \zeta + \alpha - \beta,
\end{equation*} 
then
\begin{equation*}
\begin{aligned}
    P_t (u_t - \tilde{u}_t) 
    =& \int_0^t P_s\dif u_s
    = \int_0^t (1-e^{-(\lambda_2+\beta-\alpha)s})b e^{-\alpha s}\dif t
    = b \int_0^t e^{-\alpha s} - e^{-(\beta+\lambda_2)s} \dif s \\
    =& b \left(\frac{1}{\alpha}(1-e^{-\alpha t}) - \frac{1}{\lambda_2+\beta}(1-e^{-(\lambda_2+\beta)t}) \right) 
    = u_t - \hat{v}_{t,\lambda_2}.
\end{aligned}
\end{equation*}
Thus
\begin{equation*}
    \tilde{w}_t - \hat{w}_{t,\lambda_2} \le \tilde{u}_t - \hat{v}_{t,\lambda_2} = \tilde{u}_t - u_t + P_t (u_t - \tilde{u}_t) = (1-P_t)(\tilde{u}_t - u_t).
\end{equation*}
Similarly, since
\begin{equation*}
\begin{aligned}
    P_t (v_t - \tilde{v}_t) 
    =& \int_0^t P_s\dif v_s
    = \int_0^t (1-e^{-(\lambda_1-\beta+\alpha)s})b e^{-\beta s}\dif t
    = b \int_0^t e^{-\beta s} - e^{-(\alpha+\lambda_1)s} \dif s \\
    =& b \left(\frac{1}{\beta}(1-e^{-\beta t}) - \frac{1}{\lambda_1+\alpha}(1-e^{-(\lambda_1+\alpha)t}) \right) 
    = v_t - \hat{u}_{t,\lambda_1},
\end{aligned}
\end{equation*}
we can obtain a lower bound as
\begin{equation*}
    \tilde{w}_t - \hat{w}_{t,\lambda_1} \ge \tilde{v}_t - \hat{u}_{t,\lambda_1} = \tilde{v}_t - v_t + P_t (v_t-\tilde{v}_t) = (1-P_t)(\tilde{v}_t - v_t).
\end{equation*}
These inequalities give us
\begin{equation*}
    \hat{w}_{t,\lambda_1} + (1-P_t)(\tilde{v}_t-v_t) \le \tilde{w}_t \le \hat{w}_{t,\lambda_2} + (1-P_t)(\tilde{u}_t-u_t),
\end{equation*}
which proves the continuous version of Theorem~\ref{thm:sgd-general}.

\section{Technical Lemmas}\label{sec:lemmas}

\begin{lem}\label{thm:sufficient-condition}
Consider two series $\set{x_k}_{k=0}^\infty,\ \set{\hat{x}_k}_{k=0}^\infty$, and a weighting scheme $\set{p_k}_{k=0}^\infty$ such that $\sum_{k=0}^\infty p_k = 1,\ p_k\ge 0$, $P_k = \sum_{i=1}^k p_i$.
Let $\tilde{x}_k := P_k^{-1}\sum_{i=0}^k p_i x_i$. 
Suppose $x_0 = \hat{x}_0 = 0$.
Suppose the weighting scheme $P_k$ satisfies
\begin{equation*}
    \hat{x}_{k+1} - \hat{x}_{k} = (1-P_{k})(x_{k+1} - x_{k}),\quad k\ge 0.
\end{equation*}
Then we have
\begin{equation*}
    P_k (x_k-\tilde{x}_k) = x_k - \hat{x}_k,\quad k\ge 0,
\end{equation*}
and
\begin{equation*}
    \hat{x}_k - \tilde{x}_k = \left( 1-P_k \right)(x_k-\tilde{x}_k),\quad k\ge 0.
\end{equation*}

More generally, the weighting scheme $\set{p_k}_{k=0}^\infty$ could be a series of positive semi-definite matrix where
\begin{equation*}
    \lim_{k\to +\infty} P_k = I,\quad 0\preceq P_k \preceq I,\quad p_k = P_k - P_{k-1}.
\end{equation*}
\end{lem}

\begin{proof}
By definition we know $p_0=P_0,\ p_k = P_k - P_{k-1},\ k\ge 1$, and
\begin{align*}
    P_k \tilde{x}_k 
    =& \sum_{i=1}^k p_i x_i = \sum_{i=1}^k(P_i-P_{i-1}) x_i = \sum_{i=1}^k P_i x_i - \sum_{i=1}^k P_{i-1}x_i \\
    =& P_k x_k + \sum_{i=1}^k P_{i-1} x_{i-1} - \sum_{i=1}^k P_{i-1}x_i = P_k x_k - \sum_{i=1}^k P_{i-1} (x_i - x_{i-1}).
\end{align*}
Therefore
\begin{align*}
    P_k (x_k - \tilde{x}_k) 
    =& \sum_{i=1}^k P_{i-1} (x_i - x_{i-1}) 
    = \sum_{i=1}^k (x_i-x_{i-1}) - \sum_{i=1}^k (1-P_{i-1}) (x_i - x_{i-1}) \\
    =& x_k - \sum_{i=1}^k (1-P_{i-1}) (x_i - x_{i-1}).
\end{align*}
Now use the assumption, we obtain
\begin{equation*}
    P_k (x_k - \tilde{x}_k) = x_k - \sum_{i=1}^k (\hat{x}_i - \hat{x}_{i-1}) = x_k - \hat{x}_k, \quad k\ge 1.
\end{equation*}
Thus we have
\begin{equation*}
    \hat{x}_k - \tilde{x}_k = x_k - P_k (x_k - \tilde{x}_k) - \tilde{x}_k = \left( 1-P_k \right)(x_k - \tilde{x}_k),\quad k\ge 1.
\end{equation*}
One can directly verify that the above equation also holds for $k=0$, which concludes our proof.
\end{proof}

\begin{lem}\label{thm:1dim-alpha-beta}
Let $x\in \Rbb$.
Let $f(x)$ be $\alpha$-strongly convex and $\beta$-smooth, $0<\alpha\le \beta$. 
Let $f(x)$ be lower bounded, then $x_* = \argmin_{x\in\Rbb}f(x)$ exists.
Consider GD with learning rate $\eta\in (0,\frac{1}{\beta})$, the optimization path $\{x_k\}_{k=0}^{+\infty}$ is given by
\begin{equation*}
    x_{k+1} = x_k - \eta\grad f(x_k).
\end{equation*}
If $x_0 < x_*$, then we have
\begin{enumerate}
\item For all $k>0$, $x_k\in (x_0, x_*)$.
\item For all $x\in (x_0, x_*)$, we have $\beta(x-x_*) \le \grad f(x) \le \alpha(x - x_*)$.
\item For all $x\in (x_0, x_*)$, we have $\alpha (x-x_0) + \grad f(x_0) \le \grad f(x) \le \beta (x - x_0) + \grad f(x_0)$.
\end{enumerate}
Similarly if $x_0 > x_*$, then we have
\begin{enumerate}
\item For all $k>0$, $x_k\in (x_*,x_0)$.
\item For all $x\in (x_*, x_0)$, we have $\alpha(x-x_*) \le \grad f(x) \le \beta(x - x_*)$.
\item For all $x\in (x_*, x_0)$, we have $\beta (x-x_0) + \grad f(x_0) \le \grad f(x) \le \alpha (x - x_0) + \grad f(x_0)$.
\end{enumerate}
\end{lem}
\begin{proof}
We only prove Lemma~\ref{thm:1dim-alpha-beta} in case of $x_0 < x_*$. The other case is true in a similar manner.

To prove the first conclusion we only need to show that $x_0 < x_1 < x_*$,
then recursively we obtain $x_0 < x_1 < \dots < x_k < x_*$.

Note that $\grad f(x_*) = 0$.
Since $f(x)$ is $\alpha$-strongly convex and $\beta$-smooth, we have~\cite{zhou2018fenchel}
\begin{equation*}
    \alpha (x-y)^2 \le \left(\grad f(x) - \grad f(y)\right) (x-y) \le \beta (x-y)^2.
\end{equation*}
Thus $\alpha (x_*-x_0)^2 \le -\grad f(x_0) (x_* - x_0) \le \beta (x_*-x_0)^2$.
Now by the assumption that $x_0 < x_*$, we obtain $0<\alpha (x_*-x_0)\le -\grad f(x_0) \le \beta (x_* - x_0)$.
Hence 
\begin{equation*}
\begin{aligned}
    x_1 &= x_0 - \eta \grad f(x_0) > x_0 \\
    x_1 &= x_0 - \eta \grad f(x_0) < x_0 + \eta \beta (x_* - x_0) < x_0 + x_* - x_0 < x_*.
\end{aligned} 
\end{equation*}

To prove the second conclusion, recall that
$\alpha (x_*-x)^2 \le -\grad f(x) (x_* - x) \le \beta (x_*-x)^2$,
thus for $x\in (x_0, x_*)$, we obtain
$\alpha (x_*-x)\le -\grad f(x) \le \beta (x_* - x)$.

As for the third conclusion, since
$\alpha (x-x_0)^2 \le (\grad f(x)-\grad f(x_0)) (x - x_0) \le \beta (x-x_0)^2$,
thus for $x\in (x_0, x_*)$, we obtain
$\alpha (x-x_0) + \grad f(x_0) \le \grad f(x) \le \beta (x - x_0) + \grad f(x_0)$.
which completes our proof.
\end{proof}

\section{Missing proofs in main text}\label{sec:discrete-analysis}

\subsection{Proof of Theorem~\ref{thm:sgd-linear}}\label{sec:proof-sgd-linear}
\begin{proof}
The first part of the theorem is an extension of Proposition~1 and Proposition~2 in~\cite{neu2018iterate}. 
Beyond the analysis of constant learning rate in~\cite{neu2018iterate}, we show the corresponding results for adaptive learning rates.

Recall the SGD updates for linear regression problem
\begin{equation*}
    w_{k+1} = w_k - \eta_k (x_{k+1} {x_{k+1}^\top} w_k - x_{k+1}y_{k+1} ),\quad w_0 = 0.
\end{equation*}
Let
\begin{equation*}
    \Sigma = \Ebb_{x} [x {x^\top}],\quad
    a = \Ebb_{x,y}[x y],\quad 
    w_* = {\Sigma^\inv} a,\quad
    \epsilon_{k} = (\Sigma w_{k} - a) - (x_{k+1} {x_{k+1}^\top} w_k - x_{k+1}y_{k+1}),
\end{equation*}
where $\epsilon_{k}$ is the gradient noise, and $\Ebb_{k+1} [\epsilon_{k}] = 0$.
Under these notations we have
\begin{equation}\label{eq:sgd-linear}
    w_{k+1} = w_k - \eta_k (\Sigma w_k - a) + \eta_k \epsilon_{k} = w_k - \eta_k \Sigma (w_k-w_*) + \eta_k \epsilon_{k},\quad
    w_0 = 0.
\end{equation}

Similarly for linear regression with $\ell_2$-regularization, SGD takes update 
\begin{equation*}
    \hat{w}_{k+1} = \hat{w}_k - \gamma_k (x_{k+1}x_{k+1}^\top \hat{w}_k - x_{k+1}y_{k+1} + \lambda \hat{w}_k),\quad \hat{w}_0 = 0.
\end{equation*}
Let
\begin{equation*}
    \hat{w}_* = (\Sigma + \lambda I)^{-1} a,
\end{equation*}
then
\begin{equation}\label{eq:sgd-linear-regu}
    \hat{w}_{k+1} = \hat{w}_k - \gamma_k (\Sigma \hat{w}_k - a + \lambda \hat{w}_k) + \gamma_k \epsilon_{k} = \hat{w}_k - \gamma_k (\Sigma+\lambda I) (\hat{w}_k - \hat{w}_*) + \gamma_k \epsilon_{k},\quad
    \hat{w}_0 = 0.
\end{equation}

\paragraph{Expectations}
First let us compute the expectations.
For Eq.~\eqref{eq:sgd-linear}, after taking expectation at time $k+1$, we have
\begin{equation*}
    \Ebb_{k+1} [w_{k+1}] = w_k - \eta_k \Sigma (w_k - w_*).
\end{equation*}
Then recursively taking expectation at time $k,\dots, 1$, we obtain
\begin{equation*}
    \Ebb [w_{k+1}] = \Ebb[w_k] - \eta_k\Sigma(\Ebb [w_k]-w_*), \quad \Ebb [w_0] = w_0 = 0.
\end{equation*}
Solving the above recurrence relation we have
\begin{equation*}
    \Ebb [w_k] - w_* = \Pi_{i=0}^{k-1} (I-\eta_i\Sigma)(w_0 - w_*),\quad 
    w_0 = 0,\quad 
    w_* = \Sigma^{-1}a,
\end{equation*}
hence
\begin{equation*}
    \Ebb [w_{k+1}] - \Ebb[w_k] 
    = -\Pi_{i=0}^{k-1}(I-\eta_i\Sigma)\eta_k\Sigma(w_0-w_*) 
    = \Pi_{i=0}^{k-1}(I-\eta_i\Sigma)\eta_k a,\quad
    \Ebb[w_0] = 0.
\end{equation*}

In a same way we can solve Eq.~\eqref{eq:sgd-linear-regu} in expectation and obtain
\begin{equation*}
    \Ebb [\hat{w}_{k+1}] - \Ebb[\hat{w}_k] = \Pi_{i=0}^{k-1} (I-\gamma_i(\Sigma+\lambda I))\gamma_k a, \quad 
    \Ebb [\hat{w}_0] = 0.
\end{equation*}

Notice that the weighting scheme is defined by
\begin{equation*}
    P_k = 1- \Pi_{i=0}^k(1-\lambda\gamma_i),
\end{equation*}
and $1-\lambda\gamma_i = \frac{\gamma_i}{\eta_i}$,
we can directly verify that
\begin{equation*}
    \Ebb[\hat{w}_{k+1}] - \Ebb[\hat{w}_k] = (1-P_k)(\Ebb[w_{k+1}] - \Ebb[w_k]).
\end{equation*}
Thus by Lemma~\ref{thm:sufficient-condition}, we know that
\begin{equation*}
    P_k \Ebb[\tilde{w}_k] = \Ebb[\hat{w}_k] - (1-P_k)\Ebb[w_k], \quad k\ge 0.
\end{equation*}
Hence the first conclusion holds.

\paragraph{Convergence}
By assumptions we know $0<\eta\le\eta_i<\frac{1}{\beta}\le\frac{1}{\lambda_{\max}}$, where $\lambda_{\max}$ is the largest eigenvalue of $\Sigma$.
Thus 
\begin{equation*}
    \norm{\Ebb [w_k] - w_*}_2 \le \norm{\Pi_{i=0}^{k-1} (I-\eta_i\Sigma)}_2 \cdot \norm{w_0-w_*}_2 \le \norm{(I-\eta\Sigma)^k}_2 \cdot \norm{w_0-w_*}_2 \to 0,
\end{equation*}
and $\lim_{k\to +\infty}\Ebb [w_k] = w_* = \Sigma^{-1}a$.

In a similar manner, since $\gamma_i = \frac{\eta_i}{1+\eta_i\lambda}$ and $0<\eta\le\eta_i<\frac{1}{\beta}\le \frac{1}{\lambda_{\max}}$, we have $0<\frac{\eta}{1+\lambda\eta} = \gamma\le \gamma_i<\frac{1}{\beta+\lambda}\le \frac{1}{\lambda_{\max}+\lambda}$.
Thus 
\begin{equation*}
    \norm{\Ebb [\hat{w}_k] - \hat{w}_*}_2 \le \norm{\Pi_{i=0}^{k-1} (I-\gamma_i(\Sigma+\lambda I))}_2 \cdot \norm{\hat{w}_0-\hat{w}_*}_2 \le \norm{(I-\gamma(\Sigma+\lambda I))^k}_2\cdot \norm{\hat{w}_0-\hat{w}_*}_2 \to 0,
\end{equation*}
and $\lim_{k\to +\infty}\Ebb [\hat{w}_k] = \hat{w}_* = (\Sigma+\lambda I)^{-1}a$.

On the other hand, by the first conclusion we know
\begin{equation*}
    \Ebb[\hat{w}_k] - \Ebb[\tilde{w}_k] = (1-P_k)(\Ebb[w_k] - \Ebb[\tilde{w}_k]).
\end{equation*}
Since $\Ebb[w_k]$ converges, $\Ebb[\tilde{w}_k] = {P_k^\inv}\sum_{i=1}^k p_i\Ebb[w_i]$ is bounded.
Therefore
\begin{equation*}
    \norm{\Ebb[\hat{w}_k] - \Ebb[\tilde{w}_k]}_2 = (1-P_k)\norm{\Ebb[w_k] - \Ebb[\tilde{w}_k]}_2 = \Ocal(1-P_k) = \Ocal(\Pi_{i=0}^k(1-\lambda\gamma_i))\le \Ocal((1-\lambda\gamma)^k).
\end{equation*}
Hence the second claim is true.

\paragraph{Variance}
Now we turn to analyze the deviation of the averaged solution.
From Eq.~\eqref{eq:sgd-linear}, we can recursively obtain
\begin{equation*}
    w_i = \Ebb [w_i] + \xi_i,\quad 
    \xi_i = \sum_{j=0}^{i-1} \Pi_{h=j+1}^{i-1}(I-\eta_h\Sigma)\eta_j \epsilon_{j},
\end{equation*}
where we abuse the notation and let $\Pi_{h=i}^{i-1}(I-\eta_h\Sigma) = I$.

Now applying iterate averaging with respect to $p_i = \lambda\gamma_i\Pi_{h=0}^{i-1} (1-\lambda\gamma_h)$, we have
\begin{equation*}
    P_k \tilde{w}_k = \sum_{i=1}^k p_i w_i = \sum_{i=1}^k p_i \Ebb [w_i] + \sum_{i=1}^k p_i \xi_i = P_k \Ebb[\tilde{w}_k] + \sum_{i=1}^k p_i \xi_i.
\end{equation*}
We turn to calculate the noise term $\sum_{i=1}^k p_i \xi_i$.
Note that in every step, all of the matrices can be diagonalized simultaneously, thus they commute, similarly hereinafter.
\begin{equation*}
\begin{aligned}
    & \sum_{i=1}^k p_i \xi_i = \sum_{i=1}^k p_i \left( \sum_{j=0}^{i-1} \Pi_{h=j+1}^{i-1}(I-\eta_h\Sigma)\eta_j \epsilon_{j}\right) \\
    =& \sum_{j=0}^{k-1} \left( \sum_{i=j+1}^k p_i\Pi_{h=j+1}^{i-1}(I-\eta_h\Sigma)\eta_j  \right)\epsilon_j \\
    =& \sum_{j=0}^{k-1} \left( \sum_{i=j+1}^k \lambda\gamma_i\Pi_{h=0}^{i-1} (1-\lambda\gamma_h) \Pi_{h=j+1}^{i-1}(I-\eta_h\Sigma)\eta_j  \right)\epsilon_j \\
    =& \sum_{j=0}^{k-1} \left( \sum_{i=j+1}^k \lambda\gamma_i \left(\Pi_{h=0}^{j-1} (1-\lambda\gamma_h)\right) \left(\Pi_{h=j+1}^{i-1}(1-\lambda\gamma_h)(I-\eta_h\Sigma)\right)  \left((1-\lambda\gamma_j)\eta_j\right)\right) \epsilon_j \\
    =& \sum_{j=0}^{k-1} \left( \left(\Pi_{h=0}^{j-1} (1-\lambda\gamma_h) \right) \left(\sum_{i=j+1}^k \lambda\gamma_i \Pi_{h=j+1}^{i-1}\left(I-\gamma_h (\Sigma+\lambda I)\right)\right) \gamma_j \right) \epsilon_j  \\
    =& \sum_{j=0}^{k-1} A_j \epsilon_j,
\end{aligned}
\end{equation*}
where
$A_j = \gamma_j\left(\Pi_{h=0}^{j-1} (1-\lambda\gamma_h) \right) \left(\sum_{i=j+1}^k \lambda\gamma_i \Pi_{h=j+1}^{i-1}\left(I-\gamma_h (\Sigma+\lambda I)\right)\right)$.
Recall that $\epsilon_0, \epsilon_1 \dots, \epsilon_k$ is a martingale difference sequence, 
then $\sum_{i=1}^k p_i \xi_i = \sum_{j=0}^{k-1} A_j \epsilon_j$ is a martingale.
Thus
\begin{equation*}
    \tr \var{\sum_{i=1}^k p_i \xi_i} = \tr {\var{\sum_{j=0}^{k-1} A_j \epsilon_j}} = \sum_{j=0}^{k-1}\tr {\var{A_j \epsilon_j}},
\end{equation*}
where ``Var'' is the covariance of a random vector. and ``Tr'' is the trace of a matrix.

Next we bound each term in the summation as 
\begin{equation*}
    \tr {\var{A_j \epsilon_j}} = \tr {\expect{(A_j \epsilon_j)(A_j \epsilon_j)^\top}} =  \expect{\norm{A_j \epsilon}_2^2}\le \norm{A_j}_2^2\cdot \expect{\norm{\epsilon}_2^2} 
    \le \sigma^2 \norm{A_j}_2^2.
\end{equation*}
And we remain to bound $\norm{A_j}_2^2$.
Remember that $\eta \le \eta_h \le \frac{1}{\beta},\ \gamma \le \gamma_h \le \frac{1}{\lambda+\beta}$, we have
\begin{equation*}
\begin{aligned}
    & \norm{A_j}_2^2 = \norm{\gamma_j\left(\Pi_{h=0}^{j-1} (1-\lambda\gamma_h) \right) \left(\sum_{i=j+1}^k \lambda\gamma_i \Pi_{h=j+1}^{i-1}\left(I-\gamma_h (\Sigma+\lambda I)\right)\right)}_2^2 \\
    \le & \norm{\frac{1}{\lambda+\beta}\left( (1-\lambda\gamma)^{j} \right) \left(\sum_{i=j+1}^k \frac{\lambda}{\lambda+\beta} \left(I-\gamma (\Sigma+\lambda I)\right)^{i-j-1}\right)}_2^2\\
    =& \norm{\frac{\lambda}{(\lambda+\beta)^2}\left( (1-\lambda\gamma)^{j} \right) \left(\sum_{i=0}^{k-j-1} \left(I-\gamma (\Sigma+\lambda I)\right)^{i}\right)}_2^2\\
    \le & \left(\frac{\lambda}{(\lambda+\beta)^2}\left( (1-\lambda\gamma)^{j} \right) \left(\sum_{i=0}^{k-j-1} \left(1-\gamma (\alpha+\lambda)\right)^{i}\right)\right)^2\\
    \le & \left(\frac{\lambda}{(\lambda+\beta)^2}\left( (1-\lambda\gamma)^{j} \right) \left(\frac{1}{\gamma (\alpha + \lambda)}\right)\right)^2 \\
    =& \frac{\lambda^2}{\gamma^2 (\lambda+\alpha)^2(\lambda+\beta)^4}(1-\lambda\gamma)^{2j}.
\end{aligned}
\end{equation*}
The second equality holds because $\alpha \le \lambda_{min}(\Sigma)$.

Based on previous discussion we have
\begin{equation*}
\begin{aligned}
    & \tr {\var{\sum_{i=1}^k p_i \xi_i}}
    = \sum_{j=0}^{k-1}\tr {\var{A_j \epsilon_j}}
    \le \sum_{j=0}^{k-1} \sigma^2 \norm{A_j}_2^2 \\
    \le & \sum_{j=0}^{k-1}\frac{\lambda^2\sigma^2}{\gamma^2 (\lambda+\alpha)^2(\lambda+\beta)^4}(1-\lambda\gamma)^{2j}
    \le \frac{\lambda^2\sigma^2}{\gamma^2 (\lambda+\alpha)^2(\lambda+\beta)^4} \frac{1}{1-(1-\lambda\gamma)^2} \\
    = & \frac{\lambda \sigma^2}{\gamma^3(2-\lambda\gamma)(\lambda+\alpha)^2(\lambda+\beta)^4}.
\end{aligned}
\end{equation*}

Now by multivariate Chebyshev's inequality, we have
\begin{equation*}
\begin{aligned}
    & \prob{\norm{\sum_{i=1}^k p_i\xi_i}_2 \ge \epsilon } \le \frac{\tr {\var{\sum_{i=1}^k p_i \xi_i}}}{\epsilon^2}
    \le \frac{\lambda \sigma^2}{\epsilon^2\gamma^3(2-\lambda\gamma)(\lambda+\alpha)^2(\lambda+\beta)^4}
    =& \delta.
\end{aligned}
\end{equation*}
That is, with probability at least $1-\delta$, we have
\begin{equation*}
    \norm{P_k \tilde{w}_k - P_k \Ebb[\tilde{w}_k]}_2 = \norm{\sum_{i=1}^k p_i\xi_i}_2 \le \epsilon,
\end{equation*}
where
\begin{equation*}
    \epsilon = \frac{\sigma}{\gamma(\lambda+\alpha)(\lambda+\beta)^2}\sqrt{\frac{\lambda}{\delta\gamma(2-\lambda\gamma)}}.
\end{equation*}
This completes our proof.
\end{proof}

\subsection{Proof of Theorem~\ref{thm:gd-kernel}}\label{sec:proof-gd-kernel}

\begin{proof}
The derivation of kernel ridge regression can be found in~\cite{mohri2018foundations}.
We consider the following loss function of the dual problem
\begin{equation*}
    L(\alpha,\lambda) = \half \norm{y-K\alpha}_2^2 + \frac{\lambda}{2} \alpha^\top K \alpha,
\end{equation*}
where $y = (y_1,\dots, y_n)^T$ is the label set.
Then GD takes update
\begin{equation*}
    \alpha_{k+1}  = \alpha_k - \eta_k \left( K^2 \alpha_k - Ky + \lambda K \alpha_k \right),\quad \alpha_0 = 0.
\end{equation*}
Let $\alpha_* = (K+\lambda I)^{-1} y$, then
\begin{equation*}
    \alpha_{k+1} - \alpha_* = \left( I - \eta_k (K^2 + \lambda K) \right) (\alpha_k - \alpha_*),
\end{equation*}
thus
\begin{equation*}
    \alpha_{k+1} - \alpha_* 
    = \Pi_{i=0}^k \left( I - \eta_i (K^2 + \lambda K) \right) (\alpha_0 - \alpha_*),
\end{equation*}
and
\begin{equation*}
    \alpha_{k+1} - \alpha_k 
    = \Pi_{i=0}^{k-1} \left( I - \eta_i (K^2 + \lambda K) \right) \cdot \eta_k (K^2 + \lambda K) \cdot (K+\lambda I)^{-1}y
    = \Pi_{i=0}^{k-1} \left( I - \eta_i (K^2 + \lambda K) \right) \eta_k K y.
\end{equation*}

Similarly for $\hat{\alpha}_k$, i.e., the GD path for $L(\hat{\alpha}, \hat{\lambda})$ with learning rate $\gamma_k$, we have
\begin{equation*}
    \hat{\alpha}_{k+1} - \hat{\alpha}_k = \Pi_{i=0}^{k-1} \left( I - \gamma_i (K^2 + \hat{\lambda} K) \right) \gamma_k K y.
\end{equation*}
We emphasize that the generalized learning rate
$\gamma_k = \left( I + (\hat{\lambda} - \lambda)\eta_k K \right)^{-1} \eta_k$
commutes with $K$.
And 
\begin{equation*}
    I - \gamma_k (\hat{\lambda} - \lambda)K = \frac{\gamma_k}{\eta_k}.    
\end{equation*}
Thus for the generalized weighting scheme $P_K = 1- \Pi_{i=0}^k(\gamma_i / \eta_i)$
we have
\begin{align*}
    & (1-P_k)(\alpha_{k+1} - \alpha_k) = \Pi_{i=0}^{k-1} \left(\frac{\gamma_i}{ \eta_i} \left( I-\eta_i (K^2 + \lambda K) \right) \right)\frac{\gamma_k}{\eta_k}\eta_k K y\\
    =& \Pi_{i=0}^{k-1} \left(\frac{\gamma_i}{ \eta_i} -\gamma_i (K^2 + \lambda K) \right) \gamma_k K y
    = \Pi_{i=0}^{k-1} \left(I - \gamma_i (\hat{\lambda} - \lambda)K -\gamma_i (K^2 + \lambda K) \right) \gamma_k K y \\
    =& \Pi_{i=0}^{k-1} \left( I - \gamma_i (K^2 + \hat{\lambda} K) \right) \gamma_k K y
    = \hat{\alpha}_{k+1} - \hat{\alpha}_k .
\end{align*}
Therefore by Lemma~\ref{thm:sufficient-condition} we have
\begin{equation*}
    P_k \tilde{\alpha}_k = \hat{\alpha}_k - (1-P_k) \alpha_k.
\end{equation*}

Let $\lambda_{\max}$ and $\lambda_{\min}$ be the maximal and minimal eigenvalue of $K$ respectively. Then if
\begin{equation*}
    \eta \le \eta_k \le \max\left\{ \frac{1}{\lambda_{\max} (\lambda_{\max} + \lambda)},\ \frac{1}{\lambda_{\max} (\lambda_{\max} + 2\hat{\lambda} - \lambda)} \right\},\quad
    \gamma = \left( I + (\hat{\lambda} - \lambda)\eta K \right)^{-1} \eta,
\end{equation*}
we have
\begin{equation*}
    \eta (K^2 + \lambda K) \preceq \eta_k (K^2 + \lambda K) \prec I,\quad
    \gamma (K^2 + \hat{\lambda} K) \preceq \gamma_k (K^2 + \hat{\lambda} K ) \prec I,
\end{equation*}
which guarantees the convergence of $\alpha_k$ and $\hat{\alpha}_k$.
Hence both $\alpha_k$ and $\tilde{\alpha}_k$ are bounded.
And the convergence rate is given by
\begin{align*}
    \norm{\hat{\alpha}_k - \tilde{\alpha}_k}_2 = \norm{\left(1-P_k\right) (\alpha_k - \tilde{\alpha}_k)}_2  =\bigO{\norm{1-P_k}_2} \le \bigO{\norm{\gamma/\eta}_2^k} = \bigO{ (1+(\hat{\lambda}-\lambda)\eta\lambda_{\min})^{-k} }.
\end{align*}

\end{proof}

\subsection{Proof of Theorem~\ref{thm:psgd-linear}}\label{sec:proof-psgd-linear}

\begin{proof}
Let us consider changing of variable $v_k = Q^{\half} w_k$, then
\begin{equation*}
\begin{aligned}
    & v_{k+1} = Q^{\half} w_{k+1} 
    = Q^{\half} w_k - \eta_k Q^{-\half}(x_k x_k^\top  w_k - x_k y_k) \\
    =& Q^{\half} w_k - \eta_k (Q^{-\half}x_k x_k^TQ^{-\half} Q^{\half} w_k - Q^{-\half}x_k y_k) \\
    =& v_k - \eta_k (Q^{-\half}x_k x_k^\top Q^{-\half} v_k - Q^{-\half}x_k y_k).
\end{aligned}
\end{equation*}
Similarly let $\hat{v}_k = Q^{\half} \hat{w}_k$, then
\begin{equation*}
\begin{aligned}
    & \hat{v}_{k+1} = Q^{\half} \hat{w}_{k+1} 
    = Q^{\half} \hat{w}_k - \gamma_k Q^{-\half}(x_k x_k^\top \hat{w}_k - x_k y_k - \lambda Q \hat{w}_k) \\
    =& Q^{\half} \hat{w}_k - \gamma_k (Q^{-\half}x_k x_k^\top Q^{-\half} Q^{\half} \hat{w}_k - Q^{-\half}x_k y_k - \lambda Q^{\half} \hat{w}_k) \\
    =& \hat{v}_k - \gamma_k (Q^{-\half} x_k x_k^\top Q^{-\half} \hat{v}_k - Q^{-\half}x_k y_k - \lambda \hat{v}_k).
\end{aligned}
\end{equation*}

Let us denote
\begin{equation*}
    \Sigma = \Ebb_{x} [x x^T],\quad
    a = \Ebb_{x,y}[x y],\quad 
    w_* = \Sigma^{-1} a,\quad
    \hat{w}_* = (\Sigma + \lambda I)^{-1} a,\quad
    \epsilon_{k} = (\Sigma w_{k} - a) - (x_{k+1}x_{k+1}^\top w_k - x_{k+1}y_{k+1}),
\end{equation*}
and correspondingly,
\begin{equation*}
    \Lambda = Q^{-\half} \Sigma Q^{-\half},\quad
    b = Q^{-\half} a,\quad
    v_* = Q^{-\half} w_*,\quad
    \hat{v}_* = Q^{-\half} \hat{w}_*,\quad
    \iota_{k} = Q^{-\half} \epsilon_k.
\end{equation*}
Under these notations we have
\begin{equation}\label{eq:psgd-linear}
    v_{k+1} = v_k - \eta_k (\Lambda v_k - b) + \eta_k \iota_{k},\quad
    v_0 = 0.
\end{equation}
and
\begin{equation}\label{eq:psgd-linear-regu}
    \hat{v}_{k+1} = \hat{v}_k - \gamma_k (\Lambda \hat{v}_k - b + \lambda \hat{v}_k) + \gamma_k \iota_{k},\quad
    \hat{v}_0 = 0.
\end{equation}
We can see that Eq.~\eqref{eq:psgd-linear} and Eq.~\eqref{eq:psgd-linear-regu} are exactly what we have studied in Theorem~\ref{thm:sgd-linear}.
Also by assumption we know 
\begin{equation*}
    \alpha I \preceq \Lambda \preceq \beta I.
\end{equation*}
Thus by Theorem~\ref{thm:sgd-linear} we have the following conclusions:
\begin{enumerate}
\item In expectation for any $k >0$,
\begin{equation*}
    P_k \Ebb[\tilde{v}_k] = \Ebb[\hat{v}_k] - (1-P_k)\Ebb[v_k].
\end{equation*}
\item
Both $\Ebb[v_k]$ and $\Ebb[\hat{v}_k]$ converge. 
And there exists a constant $K$ such that for all $k>K$,
\begin{equation*}
    \norm{\Ebb[\hat{v}_k] - \Ebb[\tilde{v}_k]}_2 \le \Ocal((1-\lambda\gamma)^k).
\end{equation*}
Hence the limitation of $\Ebb[\tilde{v}_k]$ exists and $\lim_{k\to\infty} \Ebb[\tilde{v}_k] = \lim_{k\to\infty} \Ebb[\hat{v}_k]$.
\item
If the noise $\iota_k$ has uniform bounded variance
\begin{equation*}
    \Ebb[\norm{\tilde{\iota}_k}_2^2] \le \norm{Q}_2 \sigma^2,\quad \forall k.
\end{equation*}
Then for $k$ large enough, with probability at least $1-\delta$, we have
\begin{equation*}
    \norm{P_k \tilde{v}_k - P_k \Ebb[\tilde{v}_k]}_2 \le \epsilon,
\end{equation*}
where
\begin{equation*}
    \epsilon = \frac{\norm{Q}_2^{\half}\sigma}{\gamma(\lambda+\alpha)(\lambda+\beta)^2}\sqrt{\frac{\lambda}{\delta\gamma(2-\lambda\gamma)}}.
\end{equation*}
\end{enumerate}

Now let $w_k = Q^{-\half} v_k$, $\hat{w}_k = Q^{-\half} \hat{v}_k$, then $\tilde{w}_k = \frac{1}{P_k} \sum_{i=1}^k p_i w_i = Q^{-\half} \frac{1}{P_k} \sum_{i=1}^k p_i v_i = Q^{-\half} \tilde{v}_k$.
Hence we have
\begin{enumerate}
\item In expectation for any $k >0$,
\begin{equation*}
    P_k \Ebb[\tilde{w}_k] = \Ebb[\hat{w}_k] - (1-P_k)\Ebb[w_k].
\end{equation*}
\item
Both $\Ebb[w_k]$ and $\Ebb[\hat{w}_k]$ converge. 
And there exists a constant $K$ such that for all $k>K$,
\begin{equation*}
    \norm{\Ebb[\hat{w}_k] - \Ebb[\tilde{w}_k]}_2 \le \Ocal((1-\lambda\gamma)^k).
\end{equation*}
Hence the limitation of $\Ebb[\tilde{w}_k]$ exists and $\lim_{k\to\infty} \Ebb[\tilde{w}_k] = \lim_{k\to\infty} \Ebb[\hat{w}_k]$.
\item
If the PSGD noise $Q^{-1} \epsilon_k$ has uniform bounded variance
\begin{equation*}
    \Ebb[\norm{Q^{-1} \epsilon_i}_2^2] \le \sigma^2,\quad \forall i.
\end{equation*}
Then for $k$ large enough, with probability at least $1-\delta$, we have
\begin{equation*}
    \norm{P_k \tilde{w}_k - P_k \Ebb[\tilde{w}_k]}_2 \le \epsilon,
\end{equation*}
where
\begin{equation*}
    \epsilon = \frac{\sigma \norm{Q^{-\half}}_2 \cdot \norm{Q^{\half}}_2 }{\gamma(\lambda+\alpha)(\lambda+\beta)^2}\sqrt{\frac{\lambda}{\delta\gamma(2-\lambda\gamma)}}
    \le \frac{\sigma \norm{Q}_2 }{\gamma(\lambda+\alpha)(\lambda+\beta)^2}\sqrt{\frac{\lambda}{\delta\gamma(2-\lambda\gamma)}}.
\end{equation*}
\end{enumerate}
Hence our claims are proved.

\end{proof}

\subsection{Proof of Theorem~\ref{thm:nsgd-linear}}
\begin{proof}\label{sec:proof-nsgd-linear}
First, provided $0<\eta<\frac{1}{\beta}<\frac{1}{\alpha}$ and $\gamma = \frac{1}{\frac{1}{\eta}+\lambda}$, we have
\begin{equation*}\label{eq:nsgd-lr}
    \frac{\eta\alpha}{\alpha+\lambda} = \frac{1}{\frac{1}{\eta}+\frac{\lambda}{\eta\alpha}} 
    < \frac{1}{\frac{1}{\eta}+\lambda} = \gamma <\frac{1}{\beta+\lambda}\le \frac{1}{\alpha+\lambda}.
\end{equation*}
Therefore $0<\frac{1-\sqrt{\gamma(\alpha+\lambda)}}{1-\sqrt{\eta\alpha}}<1$, and
\begin{equation*}
    P_k = 1 - \frac{\gamma}{\eta}\left( \frac{1-\sqrt{\gamma(\alpha+\lambda)}}{1-\sqrt{\eta\alpha}} \right)^{k-1},
    \quad p_k = P_k - P_{k-1},
\end{equation*}
is a well defined weighting scheme, i.e., $P_k$ is non-negative, non-decreasing and $\lim_{k\to\infty}P_k = 1$.

Recall the NSGD updates for linear regression problem
\begin{equation*}
    w_{k+1} = v_k - \eta (x_{k+1} x_{k+1}^\top v_k - x_{k+1} y_{k+1}), \quad v_k = w_k + \tau (w_k - w_{k-1}),\quad w_0=w_1=0,
\end{equation*}
where $\tau = \frac{1-\sqrt{\eta\alpha}}{1+\sqrt{\eta\alpha}}$.

Let 
\begin{equation*}
    \Sigma = \Ebb_x [xx^\top],\quad a=\Ebb_{x,y}[xy],\quad \epsilon_k = (\Sigma v_k - a) - (x_{k+1} x_{k+1}^\top v_k - x_{k+1}y_{k+1} ),
\end{equation*}
where $\epsilon_k$ is the gradient noise, and $\Ebb_{k+1}[\epsilon_k]=0$. Under these notations we have
\begin{equation*}
    w_{k+1} = v_k - \eta (\Sigma v_k - a) + \eta\epsilon_k, \quad v_k = w_k + \tau (w_k - w_{k-1}),\quad w_0=w_1=0.
\end{equation*}
Thus
\begin{equation}\label{eq:nsgd-linear}
    w_{k+1} = (1+\tau)(1-\eta\Sigma) w_k - \tau(1-\eta\Sigma)w_{k-1} + \eta a + \eta \epsilon_k,\quad w_0=w_1=0.
\end{equation}

Similarly for the linear regression with $\ell_2$-regularization, NSGD takes update
\begin{equation*}
    \hat{w}_{k+1} = \hat{v}_k - \gamma \left((x_{k+1} x_{k+1}^T+\lambda) \hat{v}_k-x_{k+1}y_{k+1}\right), \quad \hat{v}_k = \hat{w}_k + \hat{\tau} (\hat{w}_k-\hat{w}_{k-1}),\quad \hat{w}_0=\hat{w}_1=0,
\end{equation*}
where $\hat{\tau} = \frac{1-\sqrt{\gamma(\alpha+\lambda)}}{1+\sqrt{\gamma(\alpha+\lambda)}}$.

And we have
\begin{equation}\label{eq:nsgd-linear-regu}
    \hat{w}_{k+1} = (1+\hat{\tau})\bracket{1-\gamma(\Sigma+\lambda)} \hat{w}_k - \hat{\tau}\bracket{1-\gamma(\Sigma+\lambda)}\hat{w}_{k-1} + \gamma a + \gamma \epsilon_k,\quad \hat{w}_0=\hat{w}_1=0.
\end{equation}

\paragraph{Expectation}
First let us compute the expectations. 
Let $z_k = \Ebb[w_{k+1}] - \Ebb[w_k],\ \hat{z}_k = \Ebb[\hat{w}_{k+1}] - \Ebb[\hat{w}_k]$, we aim to show that
\begin{equation}\label{eq:lemma-condition}
    (1-P_k) z_k = \hat{z}_k, \quad k\ge 0.
\end{equation}
Then according to Lemma~\ref{thm:sufficient-condition}, we prove the first conclusion in Theorem~\ref{thm:nsgd-linear}.

We begin with solving $z_k$.

For Eq.~\eqref{eq:nsgd-linear}, taking expectation with respect to the random mini-batch sampling procedure, we have
\begin{equation*}
    \Ebb[w_{k+1}] = (1+\tau)(1-\eta\Sigma)\Ebb[w_k] - \tau(1-\eta\Sigma)\Ebb[w_{k-1}] + \eta a,\quad \Ebb[w_0]=\Ebb[w_1]=0.
\end{equation*}
Thus $z_k = \Ebb[w_{k+1}]-\Ebb[w_k]$ satisfies
\begin{equation}\label{eq:nsgd-z-iter}
    z_{k+1} = (1+\tau)(1-\eta\Sigma)z_k - \tau(1-\eta\Sigma)z_{k-1},\quad z_0=0,\quad z_1 = \eta a.
\end{equation}

Without loss of generality, let us assume $\Sigma$ is diagonal in the following. Otherwise consider its eigenvalue decomposition $\Sigma=U\Lambda U^T$, and replace $z_k$ with $U^\top z_k$.
All of the operators in the following are defined entry-wisely.

Eq.~\eqref{eq:nsgd-z-iter} defines a homogeneous linear recurrence relation with constant coefficients, which could be solved in a standard manner.
Let 
\begin{equation*}\label{eq:nsgd-AB}
    A = (1+\tau)(1-\eta\Sigma) = \frac{2(1-\eta\Sigma)}{1+\sqrt{\eta\alpha}},
    \quad B = -\tau(1-\eta\Sigma) = \frac{-(1-\sqrt{\eta\alpha})(1-\eta\Sigma)}{1+\sqrt{\eta\alpha}},
\end{equation*}
then the characteristic function of Eq.~\eqref{eq:nsgd-z-iter} is 
\begin{equation}\label{eq:nsgd-char-func}
    r^2 - Ar - B=0.
\end{equation}
Since $\Sigma$ is diagonal, $0<\eta<\frac{1}{\alpha}$, and $\alpha$ is no greater than the smallest eigenvalue of $\Sigma$, we have
\begin{equation*}
    A^2+4B = \frac{4\eta(1-\eta\Sigma)(\alpha-\Sigma)}{(1+\sqrt{\eta\alpha})^2} \le 0.
\end{equation*}
Thus the characteristic function~\eqref{eq:nsgd-char-func} has two conjugate complex roots $r_1$ and $r_2$ (they might be equal). Suppose $r_{1,2} = s\pm ti$.
Then the solution of Eq.~\eqref{eq:nsgd-z-iter} can be written as
\begin{equation*}
    z_k = 2(-B)^{\frac{k}{2}}\left( E\cos(\theta k) + F\sin(\theta k) \right),\quad k\ge 0,
\end{equation*}
where $E$ and $F$ are constants decided by initial conditions $z_0=0,\ z_1=\eta a$, and $\theta$ satisfies
\begin{equation*}
    \cos\theta = \frac{s}{\sqrt{s^2+t^2}}, \quad \sin\theta = \frac{t}{\sqrt{s^2+t^2}},\quad  r_{1,2} = s\pm ti.
\end{equation*}
Since $2s=r_1+r_2 = A, s^2+t^2 = r_1\dot r_2 = -B$, we have 
\begin{equation*}\label{eq:nsgd-theta}
    \cos\theta = \frac{A}{2\sqrt{-B}} = \sqrt{\frac{1-\eta\Sigma}{1-\eta\alpha}},\quad
    \sin\theta = \frac{\sqrt{-4B-A^2}}{2\sqrt{-B}} = \sqrt{\frac{\eta(\Sigma-\alpha)}{1-\eta\alpha}}.
\end{equation*}
Because $z_0=0, z_1 = \eta a$, we know that \begin{equation*}
    E = 0,\quad 2F = \frac{\eta a}{(-B)^{\half}\sin\theta}.
\end{equation*}
Thus
\begin{equation}\label{eq:nsgd-z-solution}
    z_k = \frac{\eta a}{\sin\theta}(-B)^{\frac{k-1}{2}}\sin(\theta k),\quad k\ge 0.
\end{equation}
where
\begin{equation*}
    B = \frac{-(1-\sqrt{\eta\alpha})(1-\eta\Sigma)}{1+\sqrt{\eta\alpha}},\quad
    \cos\theta = \sqrt{\frac{1-\eta\Sigma}{1-\eta\alpha}},\quad
    \sin\theta = \sqrt{\frac{\eta(\Sigma-\alpha)}{1-\eta\alpha}}.
\end{equation*}
One can directly verify that Eq.~\eqref{eq:nsgd-z-solution} solves the recurrence relation~\eqref{eq:nsgd-z-iter}.

Then we solve $\hat{z}_k$.

Similarly treat Eq.~\eqref{eq:nsgd-linear-regu}, we know $\hat{z}_k = \Ebb[\hat{w}_{k+1}]-\Ebb[\hat{w}_k]$ satisfies
\begin{equation*}\label{eq:nsgd-z-hat-iter}
    \hat{z}_{k+1} - (1+\hat{\tau})\bracket{1-\gamma(\Sigma+\lambda)}\hat{z}_k + \hat{\tau}\bracket{1-\gamma(\Sigma+\lambda)}\hat{z}_{k-1} = 0,\quad \hat{z}_0=0,\quad \hat{z}_1=-\gamma a.
\end{equation*}
Repeat the calculation, we obtain
\begin{equation*}\label{eq:nsgd-z-hat-solution}
    \hat{z}_k = \frac{\gamma a}{\sin\hat{\theta}}(-\hat{B})^{\frac{k-1}{2}}\sin(\hat{\theta} k),\quad k\ge 0,
\end{equation*}
where
\begin{equation*}
\begin{aligned}
    &\hat{B} = \frac{-\bracket{1-\sqrt{\gamma(\alpha+\lambda})}\bracket{1-\gamma(\Sigma+\lambda)}}{1+\sqrt{\gamma(\alpha+\lambda)}},\\
    &\cos\hat{\theta} = \sqrt{\frac{1-\gamma(\Sigma+\lambda)}{1-\gamma(\alpha+\lambda)}},\quad
    \sin\hat{\theta} = \sqrt{\frac{\gamma(\Sigma-\alpha)}{1-\gamma(\alpha+\lambda)}}.
\end{aligned}
\end{equation*}

Finally we verify the sufficient condition in Lemma~\ref{thm:sufficient-condition} (Eq.~\eqref{eq:lemma-condition}).

First we show that if $1-\lambda \gamma = \frac{\gamma}{\eta}$, we have $\theta \equiv \hat{\theta}\ (\text{mod}\ 2\pi)$.
To see this, we only need to verify that $\cos\hat{\theta} = \cos\theta,\ \sin\hat{\theta}=\sin\theta$. This is because
\begin{align*}
    &\cos\hat{\theta} = \sqrt{\frac{1-\gamma\lambda-\gamma\Sigma}{1-\gamma\lambda-\gamma\alpha}} = \sqrt{\frac{\frac{\gamma}{\eta}-\gamma\Sigma}{\frac{\gamma}{\eta}-\gamma\alpha}} = \sqrt{\frac{1-\eta\Sigma}{1-\eta\alpha}} = \cos\theta;\\
    &\sin\hat{\theta} = \sqrt{\frac{\gamma(\Sigma-\alpha)}{1-\gamma\lambda-\gamma\alpha}} = \sqrt{\frac{\gamma(\Sigma-\alpha)}{\frac{\gamma}{\eta}-\gamma\alpha}} = \sqrt{\frac{\eta(\Sigma-\alpha)}{1-\eta\alpha}} = \sin\theta.
\end{align*}
Therefore we have
\begin{equation*}
    z_k = \frac{\eta a}{\sin\theta}(-B)^{\frac{k-1}{2}}\sin(\theta k),\quad
    \hat{z}_k = \frac{\gamma a}{\sin\theta}(-\hat{B})^{\frac{k-1}{2}}\sin({\theta} k).
\end{equation*}
Since
\begin{equation*}
    1-P_k = \frac{\gamma}{\eta}\left( \frac{1-\sqrt{\gamma(\alpha+\lambda)}}{1-\sqrt{\eta\alpha}} \right)^{k-1}, \quad \frac{\gamma}{\eta} = 1-\lambda\gamma,
\end{equation*} 
we have
\begin{equation*}
\begin{aligned}
    &\frac{\eta}{\gamma}(1-P_k)(-B)^{\frac{k-1}{2}} 
    = \left(\frac{\bracket{1-\sqrt{\gamma(\alpha+\lambda)}}^2}{(1-\sqrt{\eta\alpha})^2}\cdot\frac{(1-\sqrt{\eta\alpha})(1-\eta\Sigma)}{1+\sqrt{\eta\alpha}}\right)^{\frac{k-1}{2}}\\
    =& \left( \frac{\bracket{1-\sqrt{\gamma(\alpha+\lambda)}}^2(1-\eta\Sigma)}{1-\eta\alpha} \right)^{\frac{k-1}{2}}
    = \left( \frac{\bracket{1-\sqrt{\gamma(\alpha+\lambda)}}^2(1-\gamma(\Sigma+\lambda))}{1-\gamma(\alpha+\lambda)} \right)^{\frac{k-1}{2}}\\
    =& \left( \frac{\bracket{1-\sqrt{\gamma(\alpha+\lambda)}}(1-\gamma(\Sigma+\lambda))}{1+\sqrt{\gamma(\alpha+\lambda)}} \right)^{\frac{k-1}{2}} = (-\hat{B})^{\frac{k-1}{2}}.\\
\end{aligned}
\end{equation*}
Thus $(1-P_k)z_k = \hat{z}_k$.
And according to Lemma~\ref{thm:sufficient-condition}, we have
\begin{equation*}\label{eq:nsgd-proof-error}
    \Ebb[\hat{w}_k] - \Ebb[\tilde{w}_k] = \left( 1-P_k \right)(\Ebb[w_k]-\Ebb[\tilde{w}_k]),\quad k\ge 0.
\end{equation*}
Hence the first conclusion holds.

\paragraph{Convergence}
Since $L(w)$ is $\beta$-smooth, and the corresponding learning rate $\eta<\frac{1}{\beta}$, $\Ebb[w_k]$ converges~\cite{beck2009fast}.
Similarly, $\hat{L}(\hat{w}) = L(\hat{w})+\frac{\lambda}{2}\norm{\hat{w}}_2^2$ is $(\beta+\lambda)$-smooth, and the corresponding learning rate $\gamma = \frac{1}{\frac{1}{\eta}+\lambda}< \frac{1}{\beta+\lambda}$, thus $\Ebb[\hat{w}_k]$ converges~\cite{beck2009fast}.
Specially for linear regression, these can be also verified by noticing that $0<-B<1$ because $\eta<\frac{1}{\beta}$ and
\begin{equation*}
     \sum_{i=1}^k |z_i| = \sum_{i=1}^k \left|\frac{\eta a}{\sin\theta}(-B)^{\frac{i-1}{2}}\sin(\theta i)\right| \le \sum_{i=1}^k \left|\frac{\eta a}{\sin\theta}(-B)^{\frac{i-1}{2}}\right| < +\infty,
\end{equation*}
i.e., the right hand side of the above series converge, which implies that $\Ebb[w_k] = \sum_{i=1}^k z_i$ converges absolutely, hence it converges.
In a same manner $\Ebb[\hat{w}_k]$ converges.
Thus there exist constants $M$ and $K$ such that for all $k>K$, $\norm{\Ebb[w_k]}_2 \le M$, $\norm{\Ebb[\hat{w}_k]}_2 \le M$.
Hence
\begin{equation*}
    \norm{\Ebb[\hat{w}_k] - \Ebb[\tilde{w}_k]}_2 = (1-P_k)\norm{\Ebb[w_k]-\Ebb[\hat{w}_k]}_2
    \le \frac{\gamma}{\eta}C^{k-1}\cdot 2M = \Ocal(C^k),
\end{equation*}
where $C=\frac{1-\sqrt{\gamma(\alpha+\lambda)}}{1-\sqrt{\eta\alpha}}\in (0,1)$, thus by taking limitation in both sides we obtain
\begin{equation*}
    \lim_{k\to\infty} \Ebb[\tilde{w}_k] = \lim_{k\to\infty} \Ebb[\hat{w}_k],
\end{equation*}
Hence the second conclusion holds.

\paragraph{Variance}
Next we turn to analyze the deviation of the averaged solution. 

Let $w_i = \Ebb[w_i] + \xi_i$.
Based on Eq.~\eqref{eq:nsgd-linear}, we first prove that
\begin{equation}\label{eq:nsgd-noise}
    \xi_i = \sum_{j=1}^{i-1} a_{i-j} \eta \epsilon_j,\quad i\ge 1,
\end{equation}
where
\begin{equation}\label{eq:nsgd-noise-coeff}
    a_{k+1} = A a_k + B a_{k-1},\quad a_0=0,\quad a_1=1.
\end{equation}

We prove Eq.~\eqref{eq:nsgd-noise} by mathematical induction.

For $i=1,2$, by Eq.~\eqref{eq:nsgd-linear} we know $\xi_1 = w_1 - \Ebb[w_1] = 0$ and $\xi_2 = w_2 - \Ebb[w_2] = \eta \epsilon_1$, thus Eq.~\eqref{eq:nsgd-noise} holds.
Now suppose Eq.~\eqref{eq:nsgd-noise} holds for $i-1$ and $i$, then we consider $i+1$.
In Eq.~\eqref{eq:nsgd-linear}, since $\xi_i = w_i - \Ebb[w_i]$, taking difference we have
\begin{equation*}
    \xi_{i+1} = A \xi_i + B \xi_{i-1} + \eta \epsilon_i.
\end{equation*}
Now combining the induction assumptions we have
\begin{align*}
    \xi_{i+1} 
    =& A \sum_{j=1}^{i-1} a_{i-j} \eta \epsilon_j + B \sum_{j=1}^{i-2} a_{i-j-1} \eta \epsilon_j + \eta \epsilon_i \\
    =& \sum_{j=1}^{i-2} (A a_{i-j} + B a_{i-j-1})\eta \epsilon_j + A a_1 \eta \epsilon_{i-1} + \eta \epsilon_i \\
    =& \sum_{j=1}^{i-2} a_{i-j+1}\eta \epsilon_j + a_2 \eta \epsilon_{i-1} + a_1 \eta \epsilon_i \\
    =& \sum_{j=1}^{i} a_{i-j+1}\eta \epsilon_j.
\end{align*}
Thus by mathematical induction Eq.~\eqref{eq:nsgd-noise} is true for all $i\ge 1$.

Similarly to solve $z_k$, we can solve the recurrence relation Eq.~\eqref{eq:nsgd-noise-coeff} and obtain
\begin{equation}\label{eq:nsgd-noise-coeff-solution}
    a_k = \frac{1}{\sin\theta} (-B)^{\frac{k-1}{2}} \sin(\theta k),\quad k\ge 0,
\end{equation}
where
\begin{equation*}
    B = \frac{-(1-\sqrt{\eta\alpha})(1-\eta\Sigma)}{1+\sqrt{\eta\alpha}},\quad
    \cos\theta = \sqrt{\frac{1-\eta\Sigma}{1-\eta\alpha}},\quad
    \sin\theta = \sqrt{\frac{\eta(\Sigma-\alpha)}{1-\eta\alpha}}.
\end{equation*}
Thus 
\begin{align*}
    \sqrt{-B} =& \sqrt{\frac{(1-\sqrt{\eta\alpha})(1-\eta\Sigma)}{1+\sqrt{\eta\alpha}}}
    = (1-\sqrt{\eta\alpha})\sqrt{\frac{1-\eta\Sigma}{1-\eta\alpha}},\\
    \frac{1}{\sin\theta} =& \sqrt{\frac{1-\eta\alpha}{\eta(\Sigma-\alpha)}} \preceq \sqrt{\frac{1-\eta\alpha}{\eta(\lambda_{\min}-\alpha)}} I,
\end{align*}
where $\lambda_{\min}$ is the smallest eigenvalue of $\Sigma$.

Now apply iterate averaging with respect to 
\begin{equation*}
    p_i = P_i - P_{i-1} = \frac{\gamma }{\eta} \left( \frac{\sqrt{\gamma(\alpha+\lambda)} - \sqrt{\eta\alpha}}{1-\sqrt{\eta\alpha}} \right) \left( \frac{1-\sqrt{\gamma(\alpha + \lambda)}}{1-\sqrt{\eta\alpha}} \right)^{i-2},
\end{equation*}
we have
\begin{equation*}
    P_k \tilde{w}_k = \sum_{i=1}^k p_i w_i = \sum_{i=1}^k p_i \Ebb [w_i] + \sum_{i=1}^k p_i \xi_i = P_k \Ebb[\tilde{w}_k] + \sum_{i=1}^k p_i \xi_i.
\end{equation*}

We turn to calculate the noise term $\sum_{i=1}^k p_i \xi_i$.
Note that in every step, all of the matrices can be diagonalized simultaneously, thus they commute, similarly hereinafter.
\begin{align*}
    \sum_{i=1}^k p_i \xi_i
    =& \sum_{i=1}^k p_i \sum_{j=1}^{i-1} a_{i-j} \eta \epsilon_j \\
    =& \sum_{j=1}^{k-1} \left( \sum_{i=j+1}^{k} p_{i} a_{i-j} \right) \eta \epsilon_j \\
    =& \sum_{j=1}^{k-1} A_j \epsilon_j,
\end{align*}
where $A_j = \eta \sum_{i=j+1}^{k} p_{i} a_{i-j}$.
Recall that $\epsilon_0, \epsilon_1 \dots, \epsilon_k$ is a martingale difference sequence, $\sum_{i=1}^k p_i \xi_i = \sum_{j=0}^{k-1} A_j \epsilon_j$ is a martingale.
Thus
\begin{equation*}
    \tr {\var{\sum_{i=1}^k p_i \xi_i}} = \tr {\var{\sum_{j=1}^{k-1} A_j \epsilon_j}} = \sum_{j=1}^{k-1}\tr { \var{A_j \epsilon_j}}.
\end{equation*}
Next we bound each term in the summation as 
\begin{equation*}
    \tr {\var{A_j \epsilon_j}} = \tr {\expect{ (A_j \epsilon_j)(A_j \epsilon_j)^\top }} =  \expect{\norm{A_j \epsilon}_2^2}\le \norm{A_j}_2^2\cdot \expect{\norm{\epsilon}_2^2} 
    \le \sigma^2 \norm{A_j}_2^2.
\end{equation*}
And we remain to bound $\norm{A_j}_2^2$:
\begin{align*}
    \norm{A_j}_2^2 
    =& \norm{ \eta \sum_{i=j+1}^{k} p_{i} a_{i-j} }_2^2 \\
    =& \norm{ \frac{\gamma}{\sin\theta} \frac{\sqrt{\gamma(\alpha+\lambda)} - \sqrt{\eta\alpha}}{1-\sqrt{\eta\alpha}} \sum_{i=j+1}^{k} \left( \frac{1-\sqrt{\gamma(\alpha+\lambda)}}{1-\sqrt{\eta\alpha}} \right)^{i-2} (-B)^{\frac{i-j-1}{2}} \sin(\theta (i-j)) }_2^2 \\
    \le& \norm{ \frac{\gamma}{\sin\theta}  \frac{\sqrt{\gamma(\alpha+\lambda)} - \sqrt{\eta\alpha}}{1-\sqrt{\eta\alpha}} \sum_{i=j+1}^{k} \left( \frac{1-\sqrt{\gamma(\alpha+\lambda)}}{1-\sqrt{\eta\alpha}} \right)^{i-2} \left((1-\sqrt{\eta\alpha})\sqrt{\frac{1-\eta\Sigma}{1-\eta\alpha}} \right)^{i-j-1} }_2^2 \\
    \le& \bracket{ \frac{\gamma}{\sin\theta} \frac{\sqrt{\gamma(\alpha+\lambda)} - \sqrt{\eta\alpha}}{1-\sqrt{\eta\alpha}} \sum_{i=j+1}^{k} \left( \frac{1-\sqrt{\gamma(\alpha+\lambda)}}{1-\sqrt{\eta\alpha}} \right)^{i-2} \left(1-\sqrt{\eta\alpha} \right)^{i-j-1} }^2 \\
    =& \bracket{ \frac{\gamma}{\sin\theta} \frac{\sqrt{\gamma(\alpha+\lambda)} - \sqrt{\eta\alpha}}{1-\sqrt{\eta\alpha}} \left(1-\sqrt{\eta\alpha} \right)^{1-j}  \sum_{i=j+1}^{k} \left(1-\sqrt{\gamma(\alpha+\lambda)}\right)^{i-2} }^2 \\
    \le& \left(\gamma  \sqrt{\frac{1-\eta\alpha}{\eta(\lambda_{\min}-\alpha)}}\cdot \frac{\sqrt{\gamma(\alpha+\lambda)} - \sqrt{\eta\alpha} }{\left(1-\sqrt{\eta\alpha} \right)^{j}} \cdot \frac{\left(1-\sqrt{\gamma(\alpha+\lambda)}\right)^{j-1}}{\sqrt{\gamma(\alpha+\lambda)}} \right)^2 \\
    =&  \frac{\gamma (1-\eta\alpha) \left( \sqrt{\gamma(\alpha+\lambda)} - \sqrt{\eta\alpha} \right)^2 }{\eta(\lambda_{\min} -\alpha) (\alpha+\lambda) \left(1-\sqrt{\gamma(\alpha+\lambda)}\right)^2 } \left(\frac{ 1-\sqrt{\gamma(\alpha+\lambda)}}{1-\sqrt{\eta\alpha}} \right)^{2j}.
\end{align*}
The first inequality is because $\sin(\theta (i-j)) \le 1$, and the second inequality is because $\alpha< \lambda_{\min}(\Sigma)$.

Based on previous discussion we have
\begin{equation*}
\begin{aligned}
    & \tr \var{\sum_{i=1}^k p_i \xi_i}
    = \sum_{j=1}^{k-1}\tr \var{A_j \epsilon_j}
    \le \sum_{j=1}^{k-1} \sigma^2 \norm{A_j}_2^2 \\
    \le & \sum_{j=1}^{k-1}\frac{\sigma^2 \gamma (1-\eta\alpha) \left( \sqrt{\gamma(\alpha+\lambda)} - \sqrt{\eta\alpha} \right)^2 }{\eta(\lambda_{\min} -\alpha) (\alpha+\lambda) \left(1-\sqrt{\gamma(\alpha+\lambda)}\right)^2 } 
    \left(\frac{ 1-\sqrt{\gamma(\alpha+\lambda)}}{1-\sqrt{\eta\alpha}} \right)^{2j} \\
    \le& \frac{\sigma^2 \gamma (1-\eta\alpha) \left( \sqrt{\gamma(\alpha+\lambda)} - \sqrt{\eta\alpha} \right)^2 }{\eta(\lambda_{\min} -\alpha) (\alpha+\lambda) \left(1-\sqrt{\gamma(\alpha+\lambda)}\right)^2 }
    \cdot \frac{\left(\frac{ 1-\sqrt{\gamma(\alpha+\lambda)}}{1-\sqrt{\eta\alpha}} \right)^{2}}{1-\left(\frac{ 1-\sqrt{\gamma(\alpha+\lambda)}}{1-\sqrt{\eta\alpha}} \right)^{2}}  \\
    \le& \frac{\sigma^2 \gamma (1-\eta\alpha) \left( \sqrt{\gamma(\alpha+\lambda)} - \sqrt{\eta\alpha} \right)^2 }{\eta(\lambda_{\min} -\alpha) (\alpha+\lambda) \left(1-\sqrt{\gamma(\alpha+\lambda)}\right)^2 }
    \cdot \frac{\left( 1-\sqrt{\gamma(\alpha+\lambda)} \right)^{2}}{\bracket{2-\sqrt{\eta\alpha} - \sqrt{\gamma(\alpha+\lambda)}} \bracket{\sqrt{\gamma(\alpha+\lambda)}-\sqrt{\eta\alpha}}} \\
    =& \frac{\sigma^2 \gamma (1-\eta\alpha) \left( \sqrt{\gamma(\alpha+\lambda)} - \sqrt{\eta\alpha} \right) }{\eta(\lambda_{\min} -\alpha) (\alpha+\lambda) \bracket{2-\sqrt{\eta\alpha} - \sqrt{\gamma(\alpha+\lambda)}}}.
\end{aligned}
\end{equation*}

Now by multivariate Chebyshev's inequality, we have
\begin{equation*}
\begin{aligned}
    & \Pbb\left(\norm{\sum_{i=1}^k p_i\xi_i}_2 \ge \epsilon \right) 
    \le \frac{\tr \Var[\sum_{i=1}^k p_i \xi_i]}{\epsilon^2}
    \le \frac{\sigma^2 \gamma (1-\eta\alpha) \left( \sqrt{\gamma(\alpha+\lambda)} - \sqrt{\eta\alpha} \right) }{\epsilon^2 \eta(\lambda_{\min} -\alpha) (\alpha+\lambda) \bracket{2-\sqrt{\eta\alpha} - \sqrt{\gamma(\alpha+\lambda)}}}
    =: \delta.
\end{aligned}
\end{equation*}
That is, with probability at least $1-\delta$, we have
\begin{equation*}
    \norm{P_k \tilde{w}_k - P_k \Ebb[\tilde{w}_k]}_2 = \norm{\sum_{i=1}^k p_i\xi_i}_2 \le \epsilon,
\end{equation*}
where
\begin{equation*}
    \epsilon = \sqrt{\frac{\sigma^2 \gamma (1-\eta\alpha) \left( \sqrt{\gamma(\alpha+\lambda)} - \sqrt{\eta\alpha} \right) }{\delta \eta(\lambda_{\min} -\alpha) (\alpha+\lambda) \bracket{2-\sqrt{\eta\alpha} - \sqrt{\gamma(\alpha+\lambda)}}}}.
\end{equation*}
This completes our proof.

\end{proof}

\subsection{Proof of Theorem~\ref{thm:sgd-general}}\label{sec:proof-sgd-general}

\begin{proof}
We will prove a stronger version of Theorem~\ref{thm:sgd-general} by showing the conclusions hold for any 1-dim projection direction $v_1\in \Rbb^d$.
Concisely, given a unit vector $v_1 \in \Rbb^d$, we can extend it to a group of orthogonal basis, $v_1, v_2, \dots, v_d$.
For $w\in \Rbb^d$, we denote its decomposition as
\begin{equation*}
    w = w^{(1)}v_1 + w^{(2)}v_2 + \dots + w^{(d)}v_d,\quad w^{(i)}\in \Rbb.
\end{equation*}
Define $h(w^{(1)}) = L(w) = L(w^{(1)}v_1 + \dots + w^{(d)}v_d)$, then $\grad h(w^{(1)}) = v_1^\top \grad L(w)$.
Now for one step of GD,
\begin{equation*}\label{eq:gd-iter}
    w_{k+1} = w_k - \eta \grad L(w_k),
\end{equation*}
by multiplying $v_1$ in both sides, we obtain
\begin{equation}\label{eq:1dim-gd-iter}
    w_{k+1}^{(1)} = v_1^\top w_{k+1} = v_1^\top w_k - \eta v_1^\top \grad L(w_k) = w_k^{(1)} - \eta \grad h(w_k^{(1)}).
\end{equation}
We turn to study GD along direction $v_1$ by analyzing Eq.~\eqref{eq:1dim-gd-iter}.

Firstly $h(w^{(1)})$ is $\alpha$-strongly convex, $\beta$-smooth and lower bounded since $L(w)$ is $\alpha$-strongly convex, $\beta$-smooth, and lower bounded.
Let $w_{*}$ be the unique minimum of $L(w)$, then $w_{*}^{(1)} = v_1^\top w_*$ is the minimum of $h(w^{(1)})$.
Without loss of generality, assume
\begin{equation*}
    w_*^{(1)} > 0 = w_0^{(1)}.
\end{equation*}
Then by Lemma~\ref{thm:1dim-alpha-beta}, we know the optimization path of Eq.~\eqref{eq:1dim-gd-iter} lies between $(0, w_*^{(1)})$, and for any $v\in (0, w_*^{(1)})$, we have
\begin{equation*}
    \alpha v - b \le \grad h(v) \le \beta v - b,\quad b = -\grad h (0).
\end{equation*}
Thus for Eq.~\eqref{eq:1dim-gd-iter} we have
\begin{equation*}
\begin{aligned}
    w_{k+1}^{(1)} - w_k^{(1)} =& - \eta \grad h(w_k^{(1)}) \le -\eta (\alpha w_k^{(1)} - b),\\
    w_{k+1}^{(1)} - w_k^{(1)} =& - \eta \grad h(w_k^{(1)}) \ge -\eta (\beta w_k^{(1)} - b).
\end{aligned}
\end{equation*}
Define the following dynamics:
\begin{equation*}
    u_{k+1}^{(1)}-u_k^{(1)} = -\eta (\alpha u_k^{(1)} - b),\quad
    v_{k+1}^{(1)}-v_k^{(1)} = -\eta (\beta v_k^{(1)} - b),\quad u_0^{(1)} = v_0^{(1)} = 0.
\end{equation*}
By the discrete Gronwall's inequality~\cite{clark1987short}, we have
\begin{equation*}
    v_k^{(1)} \le w_k^{(1)} \le u_k^{(1)}.
\end{equation*}
Furthermore, $u_k^{(1)}$ and $v_k^{(1)}$ satisfy two first order recurrence relations respectively, thus they can be solved by
\begin{equation*}
    u_k^{(1)} = \eta \sum_{i=1}^k(1-\eta\alpha)^{i-1}b,\quad
    v_k^{(1)} = \eta \sum_{i=1}^k(1-\eta\beta)^{i-1}b.
\end{equation*}
Since $\eta < \frac{1}{\beta}\le \frac{1}{\alpha}$, $u_k^{(1)}$ and $v_k^{(1)}$ converge.
And $w_k^{(1)}$ also converges since $h(\cdot)$ is $\beta$-smooth convex and $\eta<\frac{1}{\beta}$.

In a same way, for the regularized path,
\begin{equation*}
    \hat{w}_{k+1,\lambda}^{(1)} = \hat{w}_{k,\lambda}^{(1)} - \gamma (\grad h(\hat{w}_{k,\lambda}^{(1)})+\lambda \hat{w}_{k,\lambda}^{(1)}), \quad \hat{w}_{0,\lambda}^{(1)} = 0,
\end{equation*}
we have
\begin{equation*}
\begin{aligned}
    \hat{w}_{k+1,\lambda}^{(1)} - \hat{w}_{k,\lambda}^{(1)} =& - \gamma (\grad h(\hat{w}_{k,\lambda}^{(1)})+\lambda \hat{w}_{k,\lambda}^{(1)}) \le -\gamma\left((\alpha+\lambda)\hat{w}_{k,\lambda}^{(1)} - b \right), \\
    \hat{w}_{k+1,\lambda}^{(1)} - \hat{w}_{k,\lambda}^{(1)} =& - \gamma (\grad h(\hat{w}_{k,\lambda}^{(1)})+\lambda \hat{w}_{k,\lambda}^{(1)}) \ge -\gamma\left((\beta+\lambda)\hat{w}_{k,\lambda}^{(1)} - b \right).
\end{aligned}
\end{equation*}
Consider the following dynamics:
\begin{equation*}
    \hat{u}_{k+1,\lambda}^{(1)} - \hat{u}_{k,\lambda}^{(1)} = -\gamma((\alpha+\lambda) \hat{u}_{k,\lambda}^{(1)} - b),\quad
    \hat{v}_{k+1,\lambda}^{(1)} - \hat{v}_{k,\lambda}^{(1)} = -\gamma((\beta+\lambda) \hat{v}_{k,\lambda}^{(1)} - b),
\end{equation*}
where $\hat{u}_{0,\lambda}^{(1)} = \hat{v}_{0,\lambda}^{(1)} = 0$.
Then by the discrete Gronwall's inequality~\cite{clark1987short} and the solution of the first order recurrence relation we obtain
\begin{equation*}
    \hat{v}_{k,\lambda}^{(1)} \le  \hat{w}_{k,\lambda}^{(1)} \le  \hat{u}_{k,\lambda}^{(1)},\quad
    \hat{u}_{k,\lambda}^{(1)} = \gamma\sum_{i=1}^k(1-\gamma(\alpha+\lambda))^{i-1}b,\quad 
    \hat{v}_{k,\lambda}^{(1)} = \gamma\sum_{i=1}^k(1-\gamma(\beta+\lambda))^{i-1}b.
\end{equation*}

Now we turn to bound the iterate averaged solution. 
Consider
\begin{equation*}
    \lambda_1 = \frac{1}{\gamma}-\frac{1}{\eta}+\beta-\alpha,\quad \lambda_2 = \frac{1}{\gamma}-\frac{1}{\eta}+\alpha-\beta,
\end{equation*}
since $\beta\ge \alpha$ and $0< \gamma < \frac{1}{\beta-\alpha+1/\eta}$ we know $\lambda_1 \ge \lambda_2 > 0$.
Notice that
\begin{equation*}
    0 < \gamma (\alpha+\lambda_2) \le \{\gamma(\alpha+\lambda_1),\ \gamma(\beta+\lambda_2)\} \le \gamma(\beta+\lambda_1) = 1-\gamma(-\frac{1}{\eta}+2\beta-\alpha)<1,
\end{equation*}
where the last inequality is because $\eta > \frac{1}{2\beta-\alpha}$.
Thus $\hat{u}_{k,\lambda_1}^{(1)}$, $\hat{u}_{k,\lambda_2}^{(1)}$, $\hat{v}_{k,\lambda_1}^{(1)}$, $\hat{v}_{k,\lambda_2}^{(1)}$ converge.
Further 
$\hat{w}_{k,\lambda_1}$ and $\hat{w}_{k,\lambda_2}$ also converge since $\gamma < \frac{1}{\beta+\lambda_1} \le \frac{1}{\beta+\lambda_2}$ and the corresponding regularized losses are $(\beta+\lambda_1)$ and  $(\beta+\lambda_2)$-smooth, respectively.

Next let us consider the weighting scheme $P_k = 1-\left(\frac{\gamma}{\eta}\right)^{k+1}$, which is well defined since $0<\gamma<\frac{1}{\beta-\alpha+1/\eta}\le\eta$.

One can directly verify that $\tilde{u}_k^{(1)} = \frac{1}{P_k}\sum_{i=1}^k p_i u_i^{(1)},\ \tilde{v}_k^{(1)} = \frac{1}{P_k}\sum_{i=1}^k p_i v_i^{(1)}$ converge, and
\begin{equation*}
    (1-P_{k})(u_{k+1}^{(1)}-u_k^{(1)})=\hat{v}_{k+1,\lambda_2}^{(1)}-\hat{v}_{k,\lambda_2}^{(1)}, \quad
    (1-P_{k})(v_{k+1}^{(1)}-v_k^{(1)})=\hat{u}_{k+1,\lambda_1}^{(1)}-\hat{u}_{k,\lambda_1}^{(1)}.
\end{equation*}
Thus according to Lemma~\ref{thm:sufficient-condition} we have
\begin{equation*}
    P_k (u_k^{(1)}-\tilde{u}_k^{(1)}) = u_k^{(1)} - \hat{v}_{k,\lambda_2}^{(1)},\quad
    P_k (v_k^{(1)}-\tilde{v}_k^{(1)}) = v_k^{(1)} - \hat{u}_{k,\lambda_1}^{(1)}.
\end{equation*}
Therefore
\begin{equation*}
\begin{aligned}
    &\tilde{w}_k^{(1)} - \hat{w}_{k,\lambda_2}^{(1)} \le \tilde{u}_k^{(1)} - \hat{v}_{k,\lambda_2}^{(1)} = \tilde{u}_k^{(1)} - u_k^{(1)} + P_k (u_k^{(1)} - \tilde{u}_k^{(1)}) = (1-P_k)(\tilde{u}_k^{(1)} - u_k^{(1)}),\\
    &\tilde{w}_k^{(1)} - \hat{w}_{k,\lambda_1}^{(1)} \ge \tilde{v}_k^{(1)} - \hat{u}_{k,\lambda_1}^{(1)} = \tilde{v}_k^{(1)} - v_k^{(1)} + P_k (v_k^{(1)} - \tilde{v}_k^{(1)}) = (1-P_k)(\tilde{v}_k^{(1)} - v_k^{(1)}),
\end{aligned}
\end{equation*}
which implies that
\begin{equation}\label{eq:thm2-box}
    \hat{w}_{k,\lambda_1}^{(1)} + (1-P_k)(\tilde{v}_k^{(1)} - v_k^{(1)}) \le \tilde{w}_k^{(1)} \le \hat{w}_{k,\lambda_2}^{(1)} + (1-P_k)(\tilde{u}_k^{(1)} - u_k^{(1)}).
\end{equation}

Note that $u_k^{(1)},\ \tilde{u}_k^{(1)},\ v_k^{(1)},\ \tilde{v}_k^{(1)},\ \hat{w}_{k,\lambda_1}^{(1)},\ \hat{w}_{k,\lambda_2}^{(1)}$ converge, therefore there is a constant $M$ controlling their $\ell_2$-norm.
Define $m_k^{(1)} = (\hat{w}_{k,\lambda_2}^{(1)} + \hat{w}_{k,\lambda_1}^{(1)}) / 2$, $d_k^{(1)} = (\hat{w}_{k,\lambda_2}^{(1)} - \hat{w}_{k,\lambda_1}^{(1)}) / 2$.
Recall that $\hat{w}_{k,\lambda_1}^{(1)}$ are the GD optimization path of a $(\alpha+\lambda_1)$-strongly convex and $(\beta+\lambda_1)$-smooth loss, thus $\hat{w}_{k,\lambda_1}^{(1)}$ converges in rate $\Ocal\left((1-\gamma(\alpha+\lambda_1))^k\right)$.
Similarly $\hat{w}_{k,\lambda_2}^{(1)}$ converges in rate $\Ocal\left((1-\gamma(\alpha+\lambda_2))^k\right)$.
Thus triangle inequality we have
\begin{equation*}
    \norm{m_k^{(1)} - m^{(1)}}_2 \le \half \norm{\hat{w}_{k,\lambda_2}^{(1)} - \hat{w}_{\infty,\lambda_2}^{(1)}}_2 + \half \norm{\hat{w}_{k,\lambda_1}^{(1)} - \hat{w}_{\infty,\lambda_1}^{(1)}}_2 \le \Ocal\left((1-\gamma(\alpha+\lambda_1))^k\right) + \Ocal\left((1-\gamma(\alpha+\lambda_2))^k\right).
\end{equation*}
\begin{equation*}
    \norm{d_k^{(1)}-d^{(1)}}_2 \le \half \norm{\hat{w}_{k,\lambda_2}^{(1)} - \hat{w}_{\infty,\lambda_2}^{(1)}}_2 + \half \norm{\hat{w}_{k,\lambda_1}^{(1)} - \hat{w}_{\infty,\lambda_1}^{(1)}}_2 \le \Ocal\left((1-\gamma(\alpha+\lambda_1))^k\right) + \Ocal\left((1-\gamma(\alpha+\lambda_2))^k\right).
\end{equation*}

By Eq.~\eqref{eq:thm2-box} we obtain
\begin{equation*}
\begin{aligned}
    \tilde{w}_k^{(1)} - m_k^{(1)} 
    &\le d_k^{(1)} + (1-P_k)(\tilde{u}_k^{(1)} - u_k^{(1)}) \le d_k^{(1)} + 2M \left(\frac{\gamma}{\eta}\right)^{k+1} \\
    &\le d^{(1)} - d^{(1)} + d_k^{(1)} + \Ocal\left(\left(\frac{\gamma}{\eta}\right)^{k}\right) \\
    \le& d^{(1)} + \Ocal\left((1-\gamma(\alpha+\lambda_1))^k\right) + \Ocal\left((1-\gamma(\alpha+\lambda_2))^k\right) + \Ocal\left(\left(\frac{\gamma}{\eta}\right)^{k}\right),
\end{aligned}
\end{equation*}
and
\begin{equation*}
\begin{aligned}
    \tilde{w}_k^{(1)} - m_k^{(1)} 
    &\ge d_k^{(1)} + (1-P_k)(\tilde{v}_k^{(1)} - v_k^{(1)}) \ge d_k^{(1)} - 2M \left(\frac{\gamma}{\eta}\right)^{k+1} \\
    &\ge d^{(1)} - d^{(1)} + d_k^{(1)} - \Ocal\left(\left(\frac{\gamma}{\eta}\right)^{k}\right) \\
    \ge& d^{(1)} - \Ocal\left((1-\gamma(\alpha+\lambda_1))^k\right) - \Ocal\left((1-\gamma(\alpha+\lambda_2))^k\right) - \Ocal\left(\left(\frac{\gamma}{\eta}\right)^{k}\right).
\end{aligned}
\end{equation*}
Thus
\begin{equation*}
    \norm{\tilde{w}_k^{(1)} - m_k^{(1)}}_2 \le d^{(1)} + \Ocal(C^k),\quad C = \max\{(1-\gamma(\alpha+\lambda_1), (1-\gamma(\alpha+\lambda_2), \frac{\gamma}{\eta}\}.
\end{equation*}
In conclusion we have
\begin{equation*}
    \norm{\tilde{w}_k^{(1)} - m^{(1)}}_2 \le  \norm{\tilde{w}_k^{(1)} - m_k^{(1)}}_2 + \norm{m_k^{(1)} - m^{(1)}}_2 \le d^{(1)} + \Ocal(C^k).
\end{equation*}

\end{proof}

\section{Experiments setups}\label{sec:exp-setup}
The code is available at \url{https://github.com/uuujf/IterAvg}.

The experiments are conducted using one GPU K80 and \texttt{PyTorch 1.3.1}.

\subsection{Two dimensional toy example}
The loss function is 
\begin{align*}
    & L(w) = \half (w-w_*)^\top \Sigma (w-w_*),\quad
    w_* = (1,1)^\top,\quad
    \Sigma = U \diag\bracket{0.1, 1} U^T,\\
    & U = \begin{pmatrix}
    \cos \theta & -\sin \theta\\
    \sin \theta & \cos \theta
    \end{pmatrix},\quad
    \theta = \frac{\pi}{3}.
\end{align*}
All the algorithms are initiated from zero.
The learning rate for the unregularized problem is $\eta=0.1$.
The hyperparameter for the vanilla/generalized $\ell_2$-regularization is $\lambda = 0.1$.
And the learning rate for the regularized problem is $\gamma = \frac{1}{\lambda + 1/\eta}$.
The preconditioning matrix is set to be $Q = \Sigma$.
We run the algorithms for $500$ iterations.
For NGD and NSGD, we set the strongly convex coefficient to be $\alpha = 0.05$.

\subsection{MNIST dataset}
\paragraph{Dataset}
\url{http://yann.lecun.com/exdb/mnist/}

\paragraph{Linear regression}
The image data is scaled to $[0,1]$. 
The label data is one-hotted.
The loss function is standard linear regression under squared loss, without bias term, $L(w) = \frac{1}{2n}\sum_{i=1}^n \norm{w^T x_i - y_i}_2^2$.
All the algorithms are initiated from zero.
The learning rate for the unregularized problem is $\eta=0.01$.
The hyperparameter for the vanilla/generalized $\ell_2$-regularizer is $\lambda = 4.0$.
And the learning rate for the regularized problem is $\gamma = \frac{1}{\lambda + 1/\eta}$.
The preconditioning matrix is set to be $Q = \frac{1}{n}\sum_{i=1}^n x_i x_i^\top$.
The batch size for the stochastic algorithms are $b=500$.
We run the algorithms for $500$ iterations.
For NGD and NSGD, we set the strongly convex coefficient to be $\alpha=1.0$.

\paragraph{Logistic regression}
The image data is scaled to $[0,1]$. 
The label data is one-hotted.
The loss function is standard logistics regression loss plus an $\ell_2$-regularization term, $L(w) = \frac{1}{n}\sum_{i=1}^n \kld{y_i}{ \sigma(w^\top x_i)} + \frac{\lambda_0}{2} \norm{w}_2^2$, where $\sigma(x)$ is the softmax function and $\lambda_0 = 1.0$.
All the algorithms are initiated from zero.
The learning rate for the unregularized problem is $\eta=0.01$.
The hyperparameter for the vanilla/generalized $\ell_2$-regularizer is $\lambda = 4.0$.
And the learning rate for the regularized problem is $\gamma = \frac{1}{\lambda + 1/\eta}$.
The preconditioning matrix is set to be $Q = \frac{1}{n}\sum_{i=1}^n x_i x_i^\top$.
The batch size for the stochastic algorithms are $b=500$.
We run the algorithms for $500$ iterations.
For NGD and NSGD, we set the strongly convex coefficient to be $\alpha=1.0$.

\subsection{CIFAR-10 and CIFARR-100 datasets}
\paragraph{Datasets}
\url{https://www.cs.toronto.edu/~kriz/cifar.html}

\paragraph{VGG-16 on CIFAR-10}
The image data is scaled to $[0,1]$ and augmented by horizontally flipping and randomly cropping. 
The label data is one-hotted.
The model is standard VGG-16 with batch normalization.
We train the model with vanilla SGD for $300$ epochs.
The batch size is $100$.
The learning rate is $0.1$, and decreased by ten times at epoch $150$ and $250$.
The weight decay is set to be $5\times 10^{-4}$.

After finishing the SGD training process, we average the checkpoints from $61$ to $300$ epoch with standard geometric distribution.
We test the success probability $p\in \{0.9999, 0.999, 0.99, 0.9\}$.
And the best one is $0.99$.

\paragraph{ResNet-18 on CIFAR-10}
The image data is scaled to $[0,1]$ and augmented by horizontally flipping and randomly cropping. 
The label data is one-hotted.
The model is standard ResNet-18.
We train the model with vanilla SGD for $300$ epochs.
The batch size is $100$.
The learning rate is $0.1$, and decreased by ten times at epoch $150$ and $250$.
The weight decay is set to be $5\times 10^{-4}$.

After finishing the SGD training process, we average the checkpoints from $61$ to $300$ epoch with standard geometric distribution.
We test the success probability $p\in \{0.9999, 0.999, 0.99, 0.9\}$.
And the best one is $0.99$.

\paragraph{ResNet-18 on CIFAR-100}
The image data is scaled to $[0,1]$ and augmented by horizontally flipping and randomly cropping. 
The label data is one-hotted.
The model is standard ResNet-18.
We train the model with vanilla SGD for $300$ epochs.
The batch size is $100$.
The learning rate is $0.1$, and decreased by ten times at epoch $150$ and $250$.
The weight decay is set to be $5\times 10^{-4}$.

After finishing the SGD training process, we average the checkpoints from $61$ to $300$ epoch with standard geometric distribution.
We test the success probability $p\in \{0.9999, 0.999, 0.99, 0.9\}$.
And the best one is $0.99$.

\paragraph{Additional experiments for deep nets without weight decay}
For ResNet-18 trained on CIFAR-10, without weight decay, and with the other setups the same, vanilla SGD has $92.95\%$ test accuracy, and our method has $93.21\%$ test accuracy. This result is consistent with the results presented in the main text.

\end{document}